\documentclass[11pt]{article}

\usepackage[utf8]{inputenc}
\pdfoutput=1
\usepackage[margin=1in,letterpaper]{geometry}
\usepackage{natbib}
\usepackage[colorlinks=true,breaklinks]{hyperref}
\usepackage[hyphenbreaks]{breakurl} 
\usepackage{color}

\usepackage{microtype}
\usepackage{graphicx}

\usepackage{booktabs} 
\usepackage{enumitem}

\usepackage{hyperref}
\usepackage{wrapfig}

\usepackage{algorithm}
\usepackage{algpseudocode}

\definecolor{darkblue}{rgb}{0.1,0.1,0.75}
\hypersetup{
  colorlinks = true,
  linkcolor  = darkblue,
  citecolor  = darkblue,
  filecolor  = darkblue,
}

\makeatletter
\newlength\aftertitskip     \newlength\beforetitskip
\newlength\interauthorskip  \newlength\aftermaketitskip

\def\maketitle{\par
 \begingroup
   \def\thefootnote{\fnsymbol{footnote}}
   \def\@makefnmark{\hbox to 4pt{$^{\@thefnmark}$\hss}}
   \@maketitle \@thanks
 \endgroup
\setcounter{footnote}{0}
 \let\maketitle\relax \let\@maketitle\relax
 \gdef\@thanks{}\gdef\@author{}\gdef\@title{}\let\thanks\relax}

\def\@startauthor{\noindent \normalsize\bf}
\def\@endauthor{}
\def\@starteditor{\noindent \small {\bf Editor:~}}
\def\@endeditor{\normalsize}
\def\@maketitle{\vbox{\hsize\textwidth
 \linewidth\hsize \vskip \beforetitskip
 {\begin{center} \LARGE\@title \par \end{center}} \vskip \aftertitskip
 {\def\and{\unskip\enspace{\rm and}\enspace}%
  \def\addr{\small\it}%
  \def\email{\hfill\small\tt}%
  \def\name{\normalsize\bf}%
  \def\AND{\@endauthor\rm\hss \vskip \interauthorskip \@startauthor}
  \@startauthor \@author \@endauthor}
}}

\makeatother

\pdfoutput=1                    

\usepackage{amsmath,amsthm}
\usepackage{amssymb,amsfonts,amsxtra,mathrsfs,bm}
\usepackage{caption}
\usepackage{subcaption}
\usepackage[normalem]{ulem}

\newcommand{\Xc}{\mathcal{X}}

\newcommand{\Hc}{\mathcal{H}}

\newcommand{\Sc}{\mathcal{S}}
\newcommand{\Yc}{\mathcal{Y}}

\newcommand{\Wc}{\mathcal{W}}

\newcommand{\R}{\mathbb{R}}

\newcommand{\B}{\mathbb{B}}
\renewcommand{\L}{\mathbb{L}}
\newcommand{\Pb}{\mathbb{P}}
\newcommand{\Kc}{\mathcal{K}}

\newcommand{\ve}{\bm{e}}
\newcommand{\vw}{\bm{w}}
\newcommand{\vx}{\bm{x}}
\newcommand{\vy}{\bm{y}}
\newcommand{\vz}{\bm{z}}
\newcommand{\vv}{\bm{v}}
\newcommand{\vu}{\bm{u}}

\newcommand{\ip}[2]{\langle {#1},\, {#2} \rangle}

\newcommand{\norm}[1]{\|{#1}\|}

\DeclareMathOperator*{\argmin}{argmin}
\DeclareMathOperator*{\argmax}{argmax}

\DeclareMathOperator{\sgn}{sgn}

\DeclareMathOperator{\acosh}{acosh}
\DeclareMathOperator{\asinh}{asinh}

\renewcommand{\H}{\mathbb{H}}

\usepackage{booktabs} 
\usepackage{makecell} 

\newtheorem{theorem}{Theorem}[section]
\newtheorem{lem}[theorem]{Lemma}

\theoremstyle{definition}

\newtheorem{ass}{Assumption}
\newtheorem{rmk}[theorem]{Remark}

\numberwithin{equation}{section}

\usepackage{tikz}
\newcommand*\circled[1]{\tikz[baseline=(char.base)]{
            \node[shape=circle,draw,inner sep=1pt] (char) {#1};}}

\title{Robust large-margin learning in hyperbolic space}
\author{\name Melanie Weber$^1$ \email{mw25@math.princeton.edu}\\
\name Manzil Zaheer$^2$ \email{manzilzaheer@google.com} \\
\name Ankit Singh Rawat$^2$ \email{ankitsrawat@google.com} \\
\name Aditya Menon$^2$ \email{adityakmenon@google.com} \\
  \name Sanjiv Kumar$^2$ \email{sanjivk@google.com}\\
  \addr{$^1$Princeton University}\\
  \addr{$^2$Google Research}
}

\usepackage{amsopn}

\begin{document}
\maketitle

\begin{abstract}
 Recently, there has been a surge of interest in representation learning in hyperbolic spaces, driven by their ability to represent hierarchical data with significantly fewer dimensions than standard Euclidean spaces.
However, the viability and benefits of hyperbolic spaces for downstream machine learning tasks have received less attention. In this paper, we present, to our knowledge, the first theoretical guarantees for learning a classifier in hyperbolic rather than Euclidean space.
Specifically, we consider the problem of learning a \emph{large-margin} classifier for data possessing a hierarchical structure.
We provide an algorithm to efficiently learn a {large-margin} hyperplane,
relying on the careful injection of \emph{adversarial examples}.
Finally, we prove that for hierarchical data that embeds well into hyperbolic space, the low embedding dimension ensures superior guarantees when learning the classifier directly in hyperbolic space.
\end{abstract}

\section{Introduction}
\label{intro}
Hyperbolic spaces have received sustained interest in recent years,
owing to their ability to compactly represent data possessing hierarchical structure (e.g., trees and graphs).
In terms of \emph{representation learning},
hyperbolic spaces offer a provable advantage over Euclidean spaces for such data: 
objects requiring an \emph{exponential} number of dimensions in Euclidean space
can be represented in a \emph{polynomial} number of dimensions in hyperbolic space~\citep{sarkar}.
This has motivated research into efficiently learning a suitable hyperbolic embedding for large-scale datasets~\citep{NK17,chamberlain,poincare-glove}.

Despite this impressive representation power, little is known about the benefits of hyperbolic spaces for downstream tasks.
For example, suppose we wish to perform classification on data that is intrinsically hierarchical.
One may na\"{i}vely ignore this structure, and use a standard Euclidean embedding and corresponding classifier (e.g., SVM).
However, can we design classification algorithms that exploit the structure of hyperbolic space,
and offer \emph{provable} benefits in terms of \emph{performance}? 
This fundamental question has received surprisingly limited attention.
While some prior work has proposed specific algorithms for learning classifiers in hyperbolic space~\citep{cho,hyp-clust}, these have been primarily empirical in nature,
and do not come equipped with theoretical guarantees on convergence and generalization.

In this paper, we take a first step towards addressing this question
for
the problem of learning a \emph{large-margin} classifier.
We provide a series of algorithms to \emph{provably} learn such classifiers in hyperbolic space,
and
establish their superiority over the classifiers na\"{i}vely learned in Euclidean space.
This shows that by using a hyperbolic space that better reflects the intrinsic
geometry of the data,
one can see gains in both representation size and performance. Specifically, our contributions are:
\begin{enumerate}[label=(\roman*),leftmargin=7mm, itemsep=0mm, partopsep=0pt,parsep=0pt]
    \item We establish that the suitable injection of \emph{adversarial examples} to \emph{gradient-based} loss minimization yields an algorithm which can \emph{efficiently} learn a \emph{large-margin} classifier (Theorem~\ref{thm:gd-conv}).
    We further establish that simply performing gradient descent or using adversarial examples alone does not suffice to efficiently yield such a classifier.
    \item we compare the Euclidean and hyperbolic approaches for hierarchical data and
analyze the trade-off between low embedding dimensions and low distortion (\emph{dimension-distortion trade-off}) when learning robust classifiers on embedded data.
For hierarchical data that embeds well into hyperbolic space, it suffices to use smaller embedding dimension while ensuring superior guarantees when we learn a classifier in hyperbolic space.
\end{enumerate}
Contribution (i) establishes that it is possible to design classification algorithms that provably converge to a \emph{large-margin} separator,
by suitably injecting adversarial examples.
Contribution (ii) shows that the adaptation of algorithms to the intrinsic geometry of the data can enable efficient utilization of the embedding space without affecting the performance.

\textbf{Related work}.
Our results can be seen as hyperbolic analogue of classic results for Euclidean spaces. The large-margin learning problem is well studied in the Euclidean setting.
Classic algorithms for learning classifiers include the perceptron~\citep{perceptron, novikoff_perceptron, Freund1999} and support vector machines~\citep{cortes}.
Robust margin-learning has been widely studied; notably by~\citet{lanckriet},~\citet{El-Ghaoui},~\citet{kim2006robust} and, recently, via adversarial approaches by~\citet{charles},~\citet{telgarsky},~\citet{Li2020Implicit} and~\citet{srebo}.
Adversarial learning has recently gained interest through efforts to train more robust deep learning systems~(see, e.g., \citep{madry,Fawzi}).

Recently, the representation of (hierarchical) data in hyperbolic space has gained a surge of interest.
The literature focuses mostly on learning representations in the Poincare~\citep{NK17,chamberlain,poincare-glove} and Lorentz~\citep{NK18} models of hyperbolic space, as well as on analyzing representation trade-offs in hyperbolic embeddings~\citep{sala,W20} . The body of work on performing downstream ML tasks in hyperbolic space is much smaller and mostly without theoretical guarantees.~\citet{hyp-clust} study hierarchical clustering in hyperbolic space.~\citet{cho} introduce a hyperbolic version of support vector machines for binary classification in hyperbolic space, albeit without theoretical guarantees.~\citet{ganea} introduce a hyperbolic version of neural networks that shows empirical promise on downstream tasks, but likewise without theoretical guarantees. {To the best of our knowledge, neither robust large-margin learning nor adversarial learning in hyperbolic spaces has been studied in the literature, including \citep{cho}. Furthermore, we are not aware of any theoretical analysis of dimension-distortion trade-offs in the related literature.}

\section{Background and notation}
\vspace{-2mm}
\begin{wrapfigure}{R}{0.25\linewidth}
\vspace{-8mm}
\centering
\includegraphics[width=2.2cm]{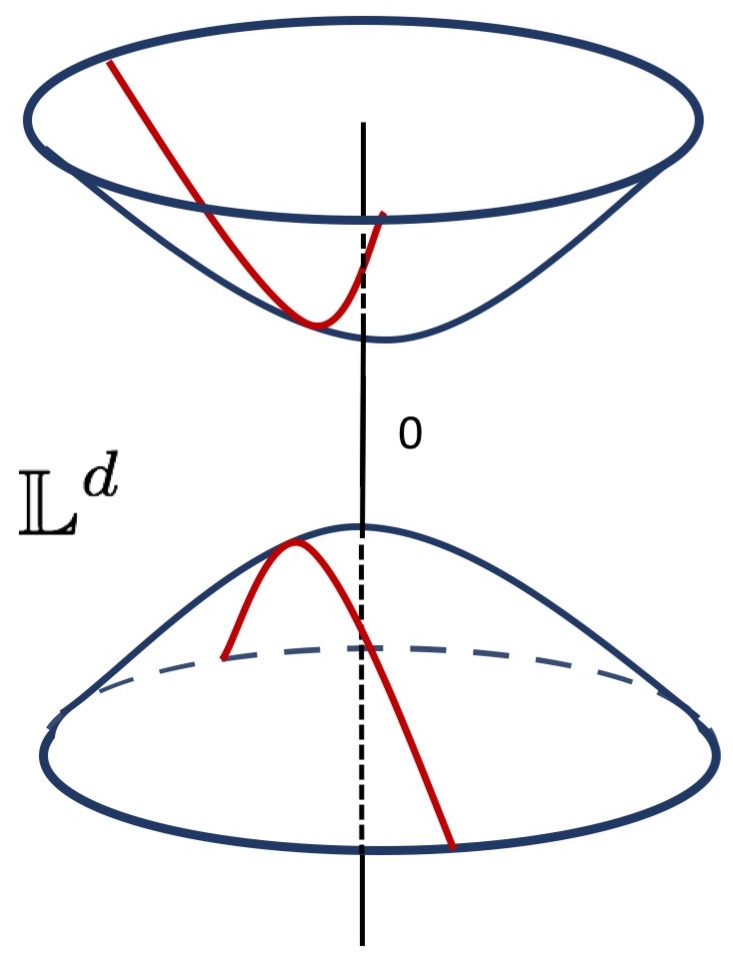}
\caption{Lorentz model (geodesics in red).}
\label{fig:hyp-model}
\vspace{-10mm}
\end{wrapfigure}
We begin by reviewing some background material on hyperbolic spaces, as well as embedding into and learning in these spaces.

\subsection{Hyperbolic space}

\label{sec:hyp-spaces-intro}
Hyperbolic spaces are smooth Riemannian manifolds
$\mathcal{M} = \H^d$ with constant negative curvature $\kappa$ and are as such locally Euclidean spaces. 
There are several equivalent models of hyperbolic space, each highlighting a different geometric aspect. In this work, we mostly consider the \emph{Lorentz model} (aka \emph{hyperboloid model}), which we briefly introduce below, with more details provided Appendix~\ref{appen:hs} (see \citep{bridson} for a comprehensive overview).

For $\vx, \vx' \in \R^{d+1}$, let $\vx \ast \vx' = x_0x'_0 - \sum_{i = 1}^{d}x_ix'_i$ denote their \emph{Minkowski product}. The Lorentz model is defined as 
$\L^d = \lbrace \vx \in \R^{d+1} : \; \vx*\vx = 1 \rbrace$ with distance measure 
$d_{\L} (\vx, \vx') = \acosh(\vx*\vx')$. Note that the distance $d_\L(\vx,\vx')$ corresponds to the length of the shortest line (\emph{geodesic}) along the manifold connecting $\vx$ and $\vx'$ (cf.~Fig.~\ref{fig:hyp-model}). We also point out that $(\L,d_{\L})$ forms a metric space. 

\subsection{Embeddability of hierarchical data}
\label{subsec: embed}
\noindent A map $\phi: X_1 \rightarrow X_2$ between metric spaces $(X_1,d_1)$ and $(X_2, d_2)$ is called an \emph{embedding}. {The \emph{multiplicative} distortion of $\phi$ is defined to be the smallest constant $c_M \geq 1$ such that, $\forall~\vx, \vx' \in X_1$,
$d_2\big(\phi(\vx),\phi(\vx'))\big) \leq d_1 (\vx,\vx') \leq c_M\cdot d_2\big(\phi(\vx),\phi(\vx')\big)$. When $c_M=1$, $\phi$ is termed an {\em isometric embedding}.} 
Since hierarchical data is tree-like, we can use classic embeddability results for trees as a reference point. 
\citet{bourgain, linial1995geometry} showed that an $N$-point metric $\Xc$ (i.e., $\vert \Xc \vert = N$) embeds into Euclidean space $\R^{O \left(\log^2 N\right)}$ with $c_M=O(\log N)$. This bound is tight for trees in the sense that embedding them in a Euclidean space (of any dimension) must have $c_m = \Omega(\log N)$~\citep{linial1995geometry}. In contrast, \citet{sarkar} showed that trees embed quasi-isometrically with $c_M = O(1+\epsilon)$ into hyperbolic space $\H^d$, even in the low-dimensional regime with the dimension as small as $d=2$.

\subsection{Classification in hyperbolic space}
We consider classification problems of the following form: $\Xc \subset \L^d$ denotes the feature space, $\Yc = \lbrace \pm 1 \rbrace$ the binary label space, and $\Wc \subset \R^{d+1}$ the model space. 
In the following, we denote the training set as $\Sc \subset \Xc \times \Yc$.

We begin by defining \emph{geodesic decision boundaries}. Consider the Lorentz space $\L^d$ with ambient space $\R^{d+1}$. Then every geodesic decision boundary is a hyperplane in $\R^d$ intersecting $\L^d$ and $\R^{d+1}$.
Further, consider the set of linear separators or \emph{decision functions} of the form
\begin{align}
\label{eq:decision_functions}
\Hc = \lbrace h_{\vw} \colon \vw \in \R^{d+1}, \vw*\vw <0 \rbrace,
~\text{  where  }~
h_{\vw}(\vx) = \begin{cases}
1, & \vw*\vx >0 \\
-1, &{\rm otherwise.} 
\end{cases}\;
\end{align}
Note that the requirement $\vw*\vw<0$ in \eqref{eq:decision_functions} ensures that the intersection of $\L^d$ and the decision hyperplane $h_{\vw}$ is not empty. The geodesic decision boundary corresponding to the decision function $h_{\vw}$ is then given by $\partial \Hc_{\vw} = \lbrace	\vz \in \L^d : \vw*\vz = 0 \rbrace.$
The distance of a point $\vx \in \L^d$ from the decision boundary $\partial\Hc_{\vw}$ can be computed as $d\big(\vx, {\partial \Hc_{\vw}}\big) = \big|\asinh \big({\vw*\vx}/{\sqrt{- \vw*\vw}}\big) \big|$~\citep{cho}.
\subsection{Large-margin classification in hyperbolic space}
In this paper, we are interested in learning a {\em large margin} classifier in a hyperbolic space. Analogous to the Euclidean setting, the natural notion of margin is the minimal distance to the decision boundary over all training samples:
\begin{align}\label{eq:hyp-margin}
{\rm margin}_{\Sc} (w) &= \inf_{(\vx,y) \in \Sc} yh_{\vw}(\vx)\cdot d(\vx, {\partial \mathcal{H}_w}) \;
= \inf_{(\vx,y) \in \Sc} \asinh \Big({y(\vw*\vx)}/{\sqrt{- \vw*\vw}}\Big) \;.
\end{align}
For \emph{large-margin classifier learning}, we aim to find 
$h_{\vw^*}$ defined by
$\vw^* = \argmax_{\vw \in \mathcal{C}}~{\rm margin}_{\Sc} (\vw)$, 
where $\mathcal{C} = \lbrace \vw \in \R^{d+1}: \; \vw*\vw <0 \rbrace$ imposes a \emph{nonconvex} constraint. This makes the problem computationally intractable using classical methods, unlike its Euclidean counterpart.

\section{Hyperbolic linear separator learning}
The first step towards the goal of learning a large-margin classifier is to establish that we 
can \emph{provably} learn \emph{some} separator.

\subsection{The hyperbolic perceptron algorithm}
\label{sec:hyp-per}
The hyperbolic perceptron (cf.~Alg.~\ref{alg:hyp-svm}) learns a binary classifier $\vw$ with respect to the Minkowski product.
This is implemented in the update rule
$\vv_t \gets \vw_t + y\hat{\vx}$,
similar to the Euclidean case.  We use the shorthand $\hat{\vx}=(x_0,-x_1, \dots, -x_n)$.
While intuitive,
it remains to establish that this algorithm converges,
i.e.,
finds a solution which correctly classifies all the training samples.
To this end, consider the following notion of \emph{hyperbolic linear separability with a margin}:
for $X, X' \subseteq \L^d$,
we say that $X$ and $X'$ are linearly separable with (hyperbolic) margin $\gamma_H$, if there exists a $\vw \in \R^{d+1}$ with $\sqrt{-\vw\ast\vw} = 1$ such that
$\vw*\vx > {\sinh}(\gamma_H)~\forall \; \vx \in X$ and $\vw*\vx' < -{\sinh}(\gamma_H)~\forall \; \vx' \in X'$. 
Assuming our training set is separable with a margin,
the hyperbolic perceptron has the following convergence guarantee:
\begin{theorem}[\citep{tabaghi}]
\label{prop:h-perceptron}
Assume that there is some $\bar{\vw} \in \R^{d+1}$ and some $\gamma_H>0$, such that $y_j (\bar{\vw}*\vx_j) \geq \sinh(\gamma_H)$~for $j=1, \dots, \vert \mathcal{S} \vert$, i.e., $\mathcal{S}$ is separable with margin $\gamma_H$. Then, Alg.~\ref{alg:hyp-svm} converges in $O\left({1}/{\sinh^2(\gamma_H)}\right)$ steps,  returning a solution with margin $\gamma_H$.
\end{theorem}
\noindent The proof of Thm.~\ref{prop:h-perceptron} is a simple adaption of
the standard proof of the Euclidean perceptron.

\begin{figure}
\begin{minipage}[t]{.48\textwidth}
\begin{algorithm}[H]\small
  \caption{Hyperbolic perceptron}
  \label{alg:hyp-svm}  
  \begin{algorithmic}[1]
   \State Initialize $\vw_0 \in \R^{d+1}$.
   \For{$t=0,1,\dots, T-1$}
   		\For {$j=1, \dots, n$}
           \If {$\sgn(\vx_j * \vw_t) \neq y_j$}
               \State $\vv_t \gets \vw_t + y_j \hat{{\vx}}_j$
           \EndIf
       \EndFor
   \EndFor
   \State Output: $\vw_T$   		
  \end{algorithmic}
\end{algorithm}
\end{minipage}\hfill
\begin{minipage}[t]{.49\textwidth}
\begin{algorithm}[H]\small
  \caption{Adversarial Training}\label{alg:adv-mf}
  \begin{algorithmic}[1] 
    \State Initialize $\vw_0 = 0$, $\Sc' = \emptyset$.
    \For {$t=0,1,\dots, T-1$}
    	\State $\Sc_t \sim \Sc$ iid with $\vert \Sc_t \vert = m$; $\Sc_t' \gets \emptyset$.
     	\For {$i=0,1,\dots, m$}
     		\State $\tilde{\vx}_i \gets \argmax_{d_\L(\vx_i,\vz)\leq \alpha} l(\vz,y_i;\vw_t)$
     	\EndFor
        \State $\Sc_t' \gets \lbrace (\tilde{\vx}_i, y_i)	\rbrace_{i=1}^m$
        \State $\Sc' \gets \Sc' \cup \Sc_t'$
        \State $\vw_{t+1} \gets \mathcal{A}(\vw_t, \Sc, \Sc')$
    \EndFor
    \State Output: $\vw_T$   
   \end{algorithmic}
  \end{algorithm}
\end{minipage}
\vspace{0mm}
\end{figure}

\subsection{The challenge of large-margin learning}
Thm.~\ref{prop:h-perceptron} establishes that the hyperbolic perceptron converges to \emph{some} linear separator. However, for the purposes of generalization, one would ideally like to converge to a \emph{large-margin} separator. As with the classic Euclidean perceptron, no such guarantee is possible for the hyperbolic perceptron; this motivates us to ask whether a suitable modification can rectify this. 

Drawing inspiration from the Euclidean setting, a natural way to proceed is to consider the use of \emph{margin losses}, such as the logistic or hinge loss.
Formally, let $l \colon \Xc \times \{ \pm 1 \} \to \mathbb{R}_+$ be a loss function:  
\begin{align}
\label{eq:f_def}
l( \vx, y; \vw ) = f( y \cdot (\vw * \vx) ),
\end{align}
where $f \colon \mathbb{R} \to \mathbb{R}_+$ is some convex, non-increasing function, e.g., the hinge loss. The empirical risk of the classifier parameterized by $\vw$ on the training set $\Sc \subset \Xc \times \{\pm 1\}$ is
$ L( \vw; \Sc ) =  \sum\nolimits_{(\vx, y) \in \Sc} l( \vx, y; \vw )/{| \Sc|}$.
Commonly, we learn a classifier by minimizing $L( \vw; \Sc )$ via gradient descent with iterates 
\begin{align}
    \label{eqn:gd-updates}
    \vw' &\gets \vw_t - {\eta}\sum\nolimits_{(\vx, y) \in \Sc} \nabla l(\vx, y;\vw_t)/ {\vert \Sc \vert} \\
    \vw_{t+1} &\gets \frac{\vw'}{\sqrt{-\vw' \ast \vw'}} \; ,
\end{align}
where $\eta > 0$ denotes the learning rate. Unfortunately, while this will yield a large-margin solution,
the following result demonstrates that the number of iterations required may be prohibitively large.

\begin{theorem}\label{prop:gd-exp}
Let $\ve_i \in \R^{d+1}$ be the $i$-th standard basis vector. Consider the training set $\Sc=\lbrace (\ve_1, 1), (-\ve_1,-1) \rbrace$ and the initialization $\vw_0=\ve_2$. Suppose $\{\vw_t\}_{t \geq 0}$ is a sequence of iterates in \eqref{eqn:gd-updates}.
Then, the number of iterations needed to achieve margin $\gamma_H$ is $\Omega(\exp(\gamma_H))$.
\end{theorem}

While this result is disheartening,fortunately,
we now present a simple resolution: by suitably adding \emph{adversarial examples}, the gradient descent converges to a large-margin solution in \emph{polynomial time}.

\section{Hyperbolic large-margin separator learning via adversarial examples}
\label{sec:adv-learning}
\setlength{\textfloatsep}{10pt}
Thm.~\ref{prop:gd-exp} reveals that gradient descent on a margin loss is insufficient to efficiently obtain a large-margin classifier.
Adapting the approach proposed in \citep{charles} for the Euclidean setting, we show how to alleviate this problem by enriching the training set with \emph{adversarial examples} before updating the classifier (cf.~Alg.~\ref{alg:adv-mf}). 
In particular, we minimize a robust loss
\begin{align}\label{eq:L-rob}
\min_{\vw \in \R^{d+1}} \; 
L_{\rm rob}(\vw; \Sc) &:= \frac{1}{\vert \mathcal{S} \vert} \sum_{(\vx,y) \in \Sc} l_{\rm rob} (\vx,y;\vw) \; \\ 
l_{\rm rob} (\vx,y;\vw) &:= \max_{\substack{\vz \in \L^d:\\ d_{\L}(\vx,\vz) \leq \alpha}} l(\vz,y;\vw)\; .
\end{align}
The problem has a minimax structure, where the outer optimization minimizes the training error over $\mathcal{S}$. 
The inner optimization generates an adversarial example by perturbing a given input feature $\vx$ on the hyperbolic manifold. Note that the magnitude of the perturbation added to the original example is bounded by $\alpha$, which we refer to as the {\em adversarial budget}. In particular, we want to construct a perturbation that maximizes the loss $l$, i.e.,
$
\tilde{\vx} \gets \argmax_{d_\L(\vx,\vz)\leq \alpha} l(\vz,y;\vw).
$
{In this paper, we restrict ourselves to only those adversarial examples that lead to misclassification with respect to the current classifier, i.e., $h_{\vw}(\vx) \neq h_{\vw}(\tilde{\vx})$ (cf.~Remark~\ref{rem:adv-cert}).}

Adversarial example $\tilde{\vx}$ can be generated efficiently by reducing the problem to an (Euclidean) linear program with a spherical constraint:
\begin{align}\label{eq:cert}
({\rm CERT}) \quad &\max_{\vz \in \R^d} \; -w_0z_0 + \sum_{i=1}^{d} w_i z_i \\
&{\rm s.t.} \;  \sum_{i=1}^{d} -x_i z_i \leq \cosh(\alpha) - x_0 z_0, \;\norm{\vz_{\backslash 0}}^2 = z_0^2 - 1 \; .
\end{align}
Importantly, as detailed in Appendix~\ref{sec:C2} and summarized next, ({\rm CERT}) can be solved in closed-form.
\begin{theorem}
\label{thm:cert}
Given the input example $(\vx, y)$, let $\vx_{\backslash 0} = (x_1,\ldots, x_d)$. We can efficiently compute a solution to {\rm CERT} or decide that no solution exists. If a solution exists, then based on a guess of $z_0$, the solution has the form $\tilde{\vx} = \left(z_0, \sqrt{z_0^2 -1} \left(b_\alpha \check{\vx} + \sqrt{1-b^2_\alpha} \check{\vx}^{\perp} \right) \right)$. Here, $b_\alpha$ depends on the adversarial budget $\alpha$, and  $\check{\vx}^{\perp}$ is a unit vector orthogonal to $\check{\vx} = -{\vx_{\backslash 0}}/{\norm{\vx_{\backslash 0}}}$ along $\vw$.
\end{theorem}
\begin{rmk}
\label{rem:adv-cert}
Note that according to Thm.~\ref{thm:cert}, it is possible that, for a particular guess of $z_0$, we may not be able to find an adversarial example $\tilde{\vx}$ that leads to a prediction that is inconsistent with $\vx$, i.e., $h_{\vw}(\vx) \neq h_{\vw}(\tilde{\vx})$. Thus, for some $t$, we may have $|\Sc_t'| < m$ in Alg.~\ref{alg:adv-mf}.  
\end{rmk}

{\renewcommand{\arraystretch}{2}
\begin{table*}[t]
\vspace{-5mm}
\footnotesize
\centering
\noindent\makebox[0.8\textwidth]{
    \begin{tabular}{lccccc}
      \toprule
      \thead{Space} & \thead{Perceptron} & \thead{Adversarial margin} &\thead{Adversarial ERM} & \thead{Adversarial GD} \\
      \midrule
      Euclidean (prior work) & $O\Big(\frac{1}{\gamma_E^2}\Big)$ & $\gamma_E - \alpha$ 
      & {$\Omega \left(\exp (d) \right)$}
      &$\Omega \left( {\rm poly}\left( \frac{1}{\gamma_E - \alpha} \right) \right)$\\
      Hyperbolic (this paper) & $O \Big(\frac{1}{\sinh^2(\gamma_H)}\Big)$ & $\gamma_H-\alpha$
      & {$\Omega \left(\exp(d)  \right)$}
      & $\Omega \Big( {\rm poly} \left( \frac{1}{\sinh(\gamma_H - \alpha)}  \right) \Big)$\\    
       \bottomrule
    \end{tabular}
    }
    \caption{Comparison between Euclidean and hyperbolic spaces for Perceptron (cf.~Alg.~\ref{alg:hyp-svm}) and adversarial training (cf.~Alg.~\ref{alg:adv-mf}). 
    Recall that $\gamma_{E/H}$, $\alpha$, and $d$ denote the (Euclidean/ hyperbolic) margin of the training data, the adversarial perturbation budget, and the underlying dimension, respectively. Note that the hyperbolic algorithms recover the rates of their Euclidean counterparts.} 
    \label{tab:perceptron}
  \end{table*}
}
\setlength{\textfloatsep}{10pt}
The minimization with respect to $\vw$ in \eqref{eq:L-rob} can be performed by an iterative optimization procedure, which generates a sequence of classifiers $\{ \vw_t \}$.
We update the classifier $\vw_t$ according to an update rule $\mathcal{A}$,
which accepts as input the current estimate of the weight vector, the original training set, and an adversarial perturbation of the training set. The update rule produces as output a weight vector which approximately minimizes the robust loss $L_{\rm rob}$ in \eqref{eq:L-rob}.

We now establish that for a gradient based update rule,
the above adversarial training procedure will efficiently converge to a large-margin solution.
Table~\ref{tab:perceptron} summarizes the results of this section and compares with the corresponding results in the Euclidean setting~\citep{charles}. 

\subsection{Fast convergence via gradient-based update}
Consider Alg.~\ref{alg:adv-mf} with 
$\mathcal{A}(\vw_t, \mathcal{S}, \mathcal{S}_t')$ being a gradient-based update
with learning rate $\eta_t > 0$:
\begin{align}\label{eq:adv-gd-1}
\vw' &\gets \vw_t - \frac{\eta_t}{\vert \mathcal{S}_t' \vert} \cdot\sum\nolimits_{(\tilde{\vx}, y) \in \Sc_t'} \nabla_{\vw_t} l(\tilde{\vx}, y ; \vw_t) \\
\vw_{t+1} &\gets \frac{\vw'}{\sqrt{- \vw' \ast \vw'}}
\; .
\end{align}

\noindent To compute the update, we need to compute gradients of the outer minimization problem, i.e., $\nabla_{\vw} \; l_{\rm rob}$ over $\mathcal{S}_t'$ (cf.~\eqref{eq:L-rob}). 
However, this function is itself a maximization problem. We therefore compute the gradient at the maximizer of this inner maximization problem. Danskin's Theorem~\citep{danskin, bertsekas} ensures that this gives a valid descent direction. 
Given the closed form expression for the adversarial example $\tilde{\vx}$ as per Thm.~\ref{thm:cert}, the gradient of the loss is 
$$\nabla_{\vw} \; l(\tilde{\vx},y;\vw) 
= f'(y(\vw*\tilde{\vx})) \cdot \nabla_{\vw} \; y(\vw*\tilde{\vx}) 
= f'(y(\vw*\tilde{\vx})) \cdot y \widehat{\tilde{\vx}}{}^T,$$
where 
$y\widehat{\tilde{\vx}}^{\rm T} = y(\tilde{x}_0, -\tilde{x}_1, \dots, -\tilde{x}_n)^{\rm T}$. With Danskin's theorem, $\nabla l(\tilde{\vx},y; \vw) \in \partial l_{\rm rob}(\vx,y; \vw)$, so we can compute the descent direction and perform the step in \eqref{eq:adv-gd-1}. We defer details to Appendix~\ref{appen:C5}. 

\subsubsection{Convergence analysis}
We now establish that the above gradient-based update converges to a large-margin solution in polynomial time.
For this analysis, we need the following assumptions:
\begin{ass}
\label{ass1}
\begin{enumerate}[leftmargin=5mm, itemsep=1mm, partopsep=0pt,parsep=0pt]
    \item The training set $\Sc$ is linearly separable with margin at least $\gamma_H$, i.e., there exists a $\bar{\vw} \in \R^{d+1}$, such that $y(\bar{\vw}*\vx) \geq \sinh(\gamma_H)$ for all $(\vx,y) \in \mathcal{S}$.
   \item There exists constants $R_x \geq 0$, such that (i) $\norm{\vx} \leq R_x^2$; (ii) all possible adversarial perturbations remain within this constraint, i.e., $\norm{\tilde{\vx}} \leq R_x^2$.
    \item the function $f(s)$, underlying the loss (cf.~\eqref{eq:f_def}), has the following properties: (i) $f(s) > 0 \; \forall \; s$; (ii) $f'(s)<0 \; \forall \; s$; (iii) $f$ is differentiable, and (iv) $f$ is $\beta$-smooth.
\end{enumerate}
\end{ass}
\noindent In the rest of the section, we work with the following hyperbolic equivalent of the \emph{logistic regression loss} that fulfills Assumption~\ref{ass1}:
\begin{align}
l(\vx,y; \vw) = \ln \left( 1 + \exp(- y(\vw*\vx) ) \right) \; .
\label{eq:hyp-logistic-reg-loss}
\end{align}
\noindent Other loss functions as well as the derivation of the hyperbolic logistic regression loss are discussed in Appendix~\ref{sec:C1}.

\noindent We first show that Alg.~\ref{alg:adv-mf} with a gradient update is guaranteed to converge to a large-margin classifier.
\begin{theorem}\label{thm:gd-conv}
With constant step size 
and $\mathcal{A}$ being the GD update
with an initialization $\vw_0$ with $\vw_0 * \vw_0 <0$, $\lim_{t \rightarrow \infty} L(\vw_t; \Sc'_t) = 0$.
\end{theorem}
\noindent The proof can be found in Appendix~\ref{appen:C5}. While this result guarantees convergence, it does not guarantee efficiency (e.g., by showing a polynomial convergence rate). The following result shows that Alg.~\ref{alg:adv-mf} with a gradient-based update obtains a max-margin classifier in polynomial time. 
\begin{theorem}
\label{conv:alg-2-main}
For a fixed constant $c \in (0, 1)$, let $\eta_t = \eta := c\cdot \frac{2\sinh^2(\gamma_H - \alpha)}{\beta \sigma_{\rm max}^2 R_x^2}$ with $\sigma_{\rm max}$ denoting an upper bound on the maximum singular value of the data matrix {$\sum_{\vx \in \Sc'}\vx\vx^T$}, and $\mathcal{A}$ the GD update as defined in \eqref{eq:adv-gd-1}.
Then, Alg.~\ref{alg:adv-mf} achieves the margin $\gamma_H - \alpha$ in $\Omega \big( {\rm poly} \left(\frac{1}{\sinh(\gamma_H - \alpha)}  \right) \big)$ steps.
\end{theorem}
\noindent Below, we briefly sketch some of the proof ideas, but defer a detailed exposition to Appendix~\ref{appen:C5} (cf. Thm.~\ref{conv:alg-2} and \ref{thm:main-iteration-complexity}).
To prove the gradient-based convergence result, we first analyze the convergence of an ``adversarial perceptron'', that resembles the adversarial GD in that it performs updates of the form $\vw_{t+1} \gets \vw_t + y \widehat{\tilde{\vx}}$. We then extend the analysis to the adversarial GD, where the convergence analysis builds on classical ideas from convex optimization.

\noindent The following auxiliary lemma relates the adversarial margin to the max-margin classifier.
\begin{lem}
Let $\bar{\vw}$ be the max-margin classifier of $\mathcal{S}$ with margin $\gamma_H$. At each iteration of Algorithm~\ref{alg:adv-mf}, $\bar{\vw}$ linearly separates $\mathcal{S} \cup \mathcal{S}'$ with margin at least $\gamma_H-\alpha$.
\end{lem}
\noindent 
A proof of the lemma can be found in Section~\ref{sec:C4}. 
With the help of this lemma, we can show the following bound on the sample complexity of the adversarial perceptron:
\begin{theorem}
 Assume that there is some $\bar{\vw} \in \R^{d+1}$ and some $\gamma_H>0$, such that $y_j (\bar{\vw}*\vx_j) \geq \sinh(\gamma_H)$ for $j=1, \dots, \vert \mathcal{S} \vert$. Then, the adversarial perceptron (with adversarial budget $\alpha$) converges after $O\left(\frac{1}{(\gamma_H - \alpha)^2}\right)$ steps, at which it has margin of at least $\gamma_H-\alpha$.
\end{theorem}
\noindent The technical proof can be found in Section~\ref{sec:C4}.
%
\subsection{On the necessity of combining gradient descent and adversarial training}
We remark here that the enrichment of the training set with adversarial examples is critical for the polynomial-time convergence.
Recall first that by Thm.~\ref{prop:gd-exp}, without adversarial training, we can construct a simple max-margin problem that cannot be solved in polynomial time. Interestingly, merely using adversarial examples by themselves does not suffice for fast convergence either.

Consider Alg.~\ref{alg:adv-mf} with an ERM as the update rule $\mathcal{A}(\vw_t, \mathcal{S}, \mathcal{S}')$. In this case, the iterate $\vw_{t+1}$ corresponds to an ERM solution for $\Sc \cup \Sc'$, i.e.,
\begin{align}
\label{eq:alg-erm}
\vw_{t+1} \gets \argmin\nolimits_{\vw} \sum\nolimits_{(\vx,y) \in \Sc \cup \Sc'} l(\vx,y; \vw) \; .
\end{align}
Let $\mathcal{S}_t = \mathcal{S}$, i.e., we utilize the full power of the adversarial training in each step. The following result reveals that even under this optimistic setting, Alg.~\ref{alg:adv-mf} may not converge to a solution with a non-trivial margin in polynomial time:
\begin{theorem}\label{prop:erm}
Suppose Alg.~\ref{alg:adv-mf} (with an ERM update) outputs a linear separator of $\Sc \cup \Sc'$. In the worst case, the number of iteration required to achieve a margin at least $\epsilon$ is $\Omega\left( \exp(d) \right)$.
\end{theorem}
We note that a similar result in the Euclidean setting appears in \citep{charles}. We establish Thm.~\ref{prop:erm} by extending the proof strategy of \citep[Thm. 4]{charles} to hyperbolic spaces. In particular, given a spherical code in $\R^d$ with $T$ codewords and $\theta \sim \sinh(\epsilon) \cosh(\alpha)$ minimum separation, we construct a training set $\Sc$ and subsequently the adversarial examples $\{\Sc'_t\}$ such that there exists a sequence of ERM solutions $\{\vw_t\}_{t \leq T}$ on $\Sc \cup \Sc'$ (cf.~\eqref{eq:alg-erm}) that has margin less than $\epsilon$. Now the result in Thm.~\ref{prop:erm} follows by utilizing a lower bound~\citep{cohn} on the size of the spherical code with $T =\Omega\left( \exp(d)  \right)$ codewords and $\theta \sim \sinh(\epsilon) \cosh(\alpha)$ minimum separation. The proof of Thm.~\ref{prop:erm} is in Appendix~\ref{sec:C3}.

\section{Dimension-distortion trade-off}
\label{sec:trade-off}
%
So far we have focused on classifying data that is given in either Euclidean spaces $\R^d$ or Lorentz space $\L^{d'}$. 
Now, consider data $(\Xc,d_{\Xc})$ with  similarity metric $d_{\Xc}$ that was embedded into the respective spaces.
We assume access to maps $\phi_E : \Xc \rightarrow \R^d$ and $\phi_H : \Xc \rightarrow \L_{+}^{d'}$
that embed $\Xc$ into the Euclidean space $\R^d$ and the upper sheet of the Lorentz space $\L_{+}^{d'}$, respectively (cf.~Remark~\ref{rmk:double-sheet}). Let $c_E$ and $c_H$ denote the multiplicative distortion induced by $\phi_E$ and $\phi_H$, respectively (cf.~\S~\ref{subsec: embed}). Upper bounds on $c_E$ and $c_H$ can be estimated based on the structure of $\Xc$ and the embedding dimensions. 
\begin{wrapfigure}{R}{0.6\linewidth}
\includegraphics[scale=0.14]{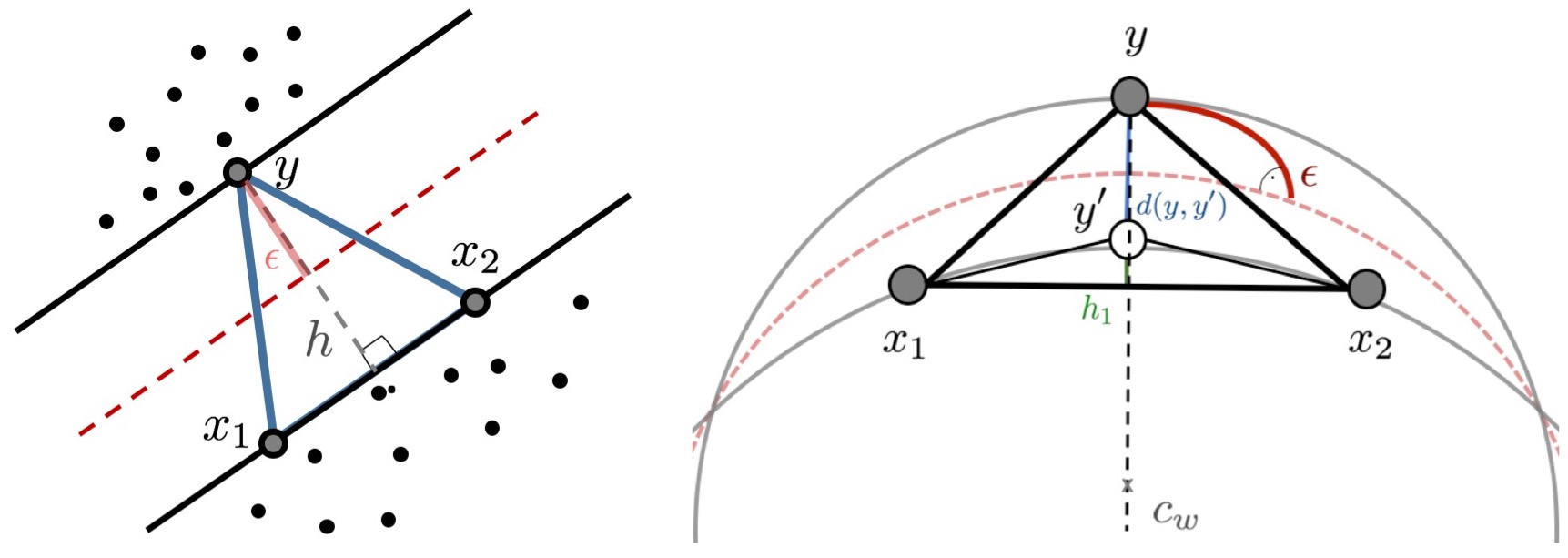}
\caption{Margin as distance between support vectors. \textbf{Left:} Euclidean. \textbf{Right:} Hyperbolic.}
\label{fig:margin}
\vspace{-3mm}
\end{wrapfigure}

In this section, we address the natural question: {\em How does the distortion $c_E, c_H$ impact our guarantees on the margin?} In the previous sections, we noticed that some of the guarantees scale with the dimension of the embedding space. Therefore, we want to analyze the trade-off between the higher distortion resulting from working with smaller embedding dimensions and the higher cost of training robust models due to working with larger embedding dimensions.

We often encounter data sets in ML applications that are intrinsically hierarchical. Theoretical results on the embeddability of trees (cf.~\S~\ref{subsec: embed}) suggest that hyperbolic spaces are especially suitable to represent hierarchical data. We therefore restrict our analysis to such data. Further, we make the following assumptions on the underlying data $\Xc$ and the embedding maps, respectively.
\begin{ass}\label{ass:dist}
(1) Both $\phi_{H}(\Xc)$ and $\phi_{E}(\Xc)$ are linearly separable in the respective spaces, and (2) $\Xc$ is hierarchical, i.e., has a partial order relation.
\end{ass}
\begin{ass}\label{ass:maps}
The maps $\phi_H, \phi_E$ preserve the partial order relation in $\Xc$ and the root is mapped onto the origin of the embedding space.
\end{ass}
Towards understanding the impact of the distortion of the embedding maps $\phi_H$ and $\phi_E$ on margin, we relate the distance between the support vectors to the size of the margin. The distortion of these distances via embedding then gives us the desired bounds on the margin. 
\subsection{Euclidean case}
\vspace{-1mm}
In the Euclidean case, we relate the distance of the support vectors to the size of the margin via triangle relations. Let $\vx,\vy \in \R^{d}$ denote support vectors, such that $\ip{\vx}{\vw} > 0$ and $\ip{\vy}{\vw} <0$ and ${\rm margin}(\vw)=\epsilon$. Note that we can rotate the decision boundary, such that the support vectors are not unique. So, without loss of generality, assume that $\vx_1, \vx_2$ are equidistant from the decision boundary and $\norm{\vw}=1$ (cf.~Fig.~\ref{fig:margin}(left)). In this setting, we show the following relation between the margin with and without the influence of distortion:
\begin{theorem}\label{prop:e-dist}
Let $\epsilon'$ and $\epsilon$ denote the margin with and without distortion, respectively. If $\mathcal{X}$ is a tree embedded into $\R^{O(\log^2 \vert \mathcal{X} \vert)}$, then $\epsilon'=O \left( {\epsilon}/{\log^3 \vert \mathcal{X} \vert}    \right)$. 
\end{theorem}
The proof of Thm.~\ref{prop:e-dist} follows from a simple side length-altitude relations in the Euclidean triangle between support vectors (cf.~Fig.~\ref{fig:margin}(left)) and a simple application of Bourgain's result on embedding trees into $\mathbb{R}^d$.
For more details see Appendix~\ref{sec:D1}.

\setlength{\textfloatsep}{10pt}
\subsection{Hyperbolic case}
As in the Euclidean case, we want to relate the margin to the pairwise distances of the support vectors. Such a relation can be constructed both in the original and in the embedding space, which allows us to study the influence of distortion on the margin in terms of $c_H$. In the following, we will work with the half-space model $\Pb^{d'}$ (cf. Appendix~\ref{sec:A1}). However, since the theoretical guarantees in the rest of the paper consider the Lorentz model $\L_{+}^{d'}$, we have to map between the two spaces. We show in Appendix~\ref{sec:D2} that such a mapping exists and preserves the Minkowski product, following ~\cite{cho}.

The hyperbolic embedding $\phi_H$ has two sources of distortion: (1) the multiplicative distortion of pairwise distances, measured by the factor ${1}/{c_H}$; and (2) the distortion of order relations, in most embedding models captured by the alignment of ranks with the Euclidean norm. 
Under Assumption~\ref{ass:maps}, order relationships are preserved and the root is mapped to the origin. Therefore, for $x \in \Xc$, the distortion on the Euclidean norms is given as  $\norm{\phi_H(x)} = d_E(\phi_H(x),\phi_H(0)) = {d_\Xc(x,0)}/c_H$,
i.e., the distortion on both pairwise distances and norms is given by a factor $1/{c_H}$.

In $\Pb^{d'}$, the decision hyperplane corresponds to a hypercircle $\Kc_w$. We express its radius $r_w$ in terms of the hyperbolic distance between a point on the decision boundary and one of the hypercircle's ideal points~\citep{cho}. The support vectors $
\vx,\vy$ lie on hypercircles $\Kc_x$ and $\Kc_y$, which correspond to the set of points of hyperbolic distance $\epsilon$ (i.e., the margin) from the decision boundary. We again assume, without loss of generality, that at least one support vector is not unique and let $\vx_1, \vx_2 \in \Kc_x$ and $\vy \in \Kc_y$ (cf.~Fig.~\ref{fig:margin}(right)). We now show that the distortion introduced by $\phi_H$ has a negligible effect on the margin.
\begin{theorem}\label{thm:h-dist-dim}
Let $\epsilon'$ and $\epsilon$ denote the margin with and without distortion, respectively. 
If $\mathcal{X}$ is a tree embedded into $\L_{+}^2$, then $\epsilon' \approx \epsilon$.
\end{theorem}
The technical proof relies on a construction that reduces the problem to Euclidean geometry via circle inversion on the decision hypercircle. We defer all details to Appendix~\ref{sec:D2}.

\section{Experiments}
\label{sec:exp}
\begin{figure*}
    \centering
    \hfill
    \includegraphics[width=0.3\textwidth]{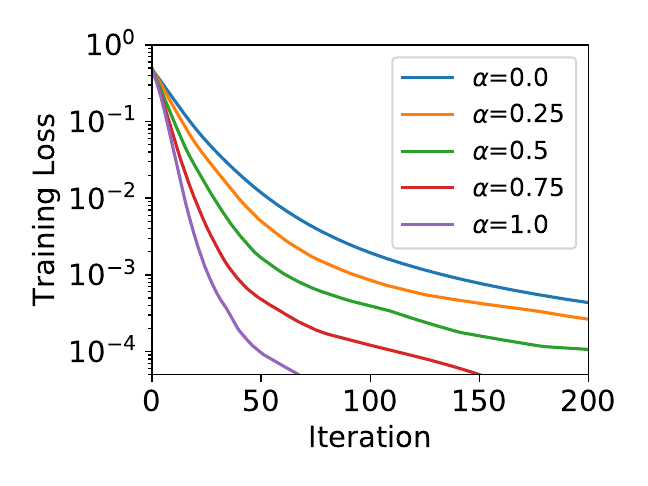}
    \hfill
    \includegraphics[width=0.3\textwidth]{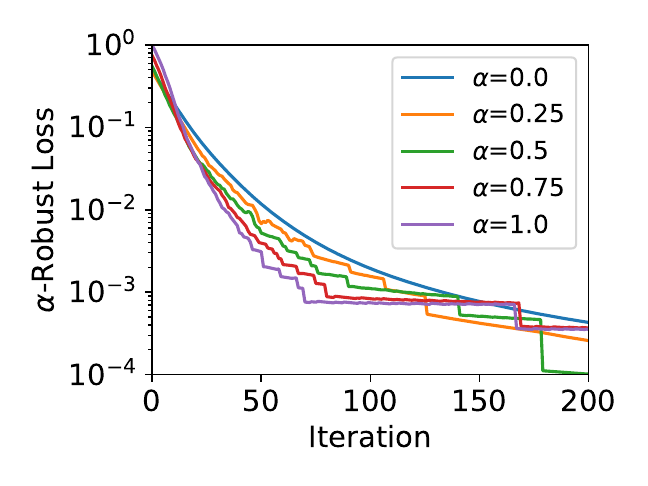}
    \hfill
    \includegraphics[width=0.3\textwidth]{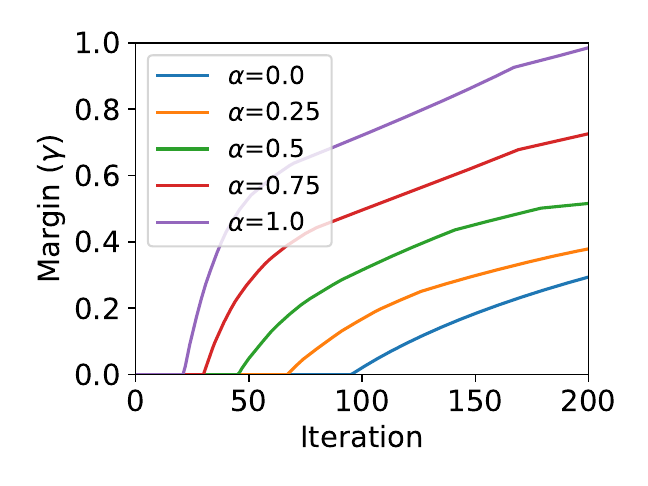}
    \hfill
    \caption{Performance of Adversarial GD. \textbf{Left:} Loss $L(\vw)$ on the original data. \textbf{Middle:} $\alpha$-robust loss $L_\alpha(\vw)$. \textbf{Right:} Hyperbolic margin $\gamma_H$.
We vary the adversarial budget $\alpha$ over $\{0, 0.25, 0.5, 0.75, 1.0\}$. Note that $\alpha=0$ corresponds to the state of the art~\citep{cho}.}
    \label{fig:adv-gd}
    \vspace{-2mm}
\end{figure*}
We now present empirical studies for hyperbolic linear separator learning to corroborate our theory. 
In particular, we evaluate our proposed Adversarial GD algorithm (\S\ref{sec:adv-learning}) on data that is linearly separable in hyperbolic space and 
compare with the state of the art~\citep{cho}.
Furthermore, we
analyze dimension-distortion trade-offs (\S\ref{sec:trade-off}). 
As in our theory, we train hyperbolic linear classifiers $\vw$ whose prediction on $\vx$ is $y = \sgn(\vw*\vx)$. 
Additional experimental results can be found in Appendix~\ref{sec:apx-exp}.

{We emphasise that 
our focus in this work is in theoretically understanding the benefits of hyperbolic spaces for classification.
The above experiments serve to illustrate our theoretical results, which as a starting point were derived for linear models and separable data.
While extensions to non-separable data and non-linear models are of  practical interest,
a detailed study is left for future work.}

\textbf{Data. }
We use the ImageNet ILSVRC 2012 dataset \cite{russakovsky2015imagenet} along with its label hierarchy from wordnet. Hyperbolic embeddings in Lorentz space are obtained for the internal label nodes and leaf image nodes using Sarkar's construction~\citep{sarkar}. For the first experiment, we verify the effectiveness of adversarial learning by picking two classes (n09246464 and n07831146), which allows for a data set that is linearly separable in hyperbolic space. In this set, there were 1,648 positive and 1,287 negative examples.
For the second experiment, to showcase better representational power of hyperbolic spaces for hierarchical data, we pick two disjoint subtrees (n00021939 and n00015388) from the hierarchy.
 
\textbf{Adversarial GD. }
In the following, we utilize the hyperbolic hinge loss~\eqref{eq:cho-hinge}, (see Appendix~\ref{sec:apx-exp} for other loss functions). 
To verify the effectiveness of adversarial training, we compute three quantities: (i) loss on original data $L(\vw)$, (ii) $\alpha$-robust loss $L_\alpha(\vw)$, and (iii) the hyperbolic margin $\gamma$. We vary the budget $\alpha$ over $\{0, 0.25, 0.5, 0.75, 1.0\}$, where $\alpha=0$ corresponds to the setup in~\citep{cho}. For a given budget, we obtain adversarial examples by solving the CERT problem (\ref{eq:cert}), which is feasible for $z_0 \in (x_0 \cosh(\alpha) - \Delta, \, x_0 \cosh(\alpha) + \Delta)$, where $\Delta =  \sqrt{(x_0^2-1)(\cosh^2(\alpha)-1)}$. We do a binary search in this range for $z_0$, solve CERT and check if we can obtain an adversarial example. We utilize the adversarial examples we find, and ignore other data points. In all experiments, we use a constant step-size $\eta_t=0.01~\forall t$. 
The results are shown in Fig.~\ref{fig:adv-gd}. As $\alpha$ increases the problem becomes harder to solve (higher training robust loss) but we achieve a better margin. Notably, we observe strong performance gains over the training procedures without adversarial examples~\citep{cho}.

\textbf{Dimensional efficiency. }
In this experiment, we illustrate the benefit of using hyperbolic space when the underlying data is truly hierarchical. To be more favourable to Euclidean setting, we subsample images from each subtree, such that in total we have 1000 vectors. We obtain Euclidean embeddings following the setup and code from~\citet{NK17}. 
The Euclidean embeddings in 16 dimensions reach a mean rank (MR) $\leq 2$, which indicates reasonable quality in preserving distance to few-hop neighbors. We observe superior classification performance at much lower dimensions by leveraging hyperbolic space (see Table~\ref{tab:my_label} in Appendix~\ref{apx:f-3}). In particular, our hyperbolic classifier achieves zero test error on 8-dimensional embeddings, whereas Euclidean logistic regression struggles even with 16-dimensional embeddings. This is consistent with our theoretical results (\S\ref{sec:trade-off}): Due to high distortion, lower-dimensional Euclidean embeddings struggle to capture the global structure among the data points that makes the data points easily separable.

\section{Conclusion and future work}

We studied the problem of learning robust classifiers with large margins in hyperbolic space. 
First, we explored multiple adversarial approaches to robust large-margin learning. 
In the second part of the paper we analyzed the role of geometry in learning robust classifiers. Here, we compared Euclidean and hyperbolic approaches with respect to the intrinsic geometry of the data. For hierarchical data that embeds 
well into hyperbolic space, the lower embedding dimension ensures superior guarantees when learning the classifier in hyperbolic space. This result suggests that it can be highly beneficial to perform downstream machine learning and
optimization tasks in a space that naturally reflects the intrinsic geometry of the data. Promising avenues for future research include (i) exploring the practicality of these results in broader machine learning and data science applications; and (ii) studying other related methods in non-Euclidean spaces, together with an evaluation of dimension-distortion trade-offs.

\bibliographystyle{plainnat}
\bibliography{ref}

\newpage
\appendix
\section{Hyperbolic Space}
\label{appen:hs}
Hyperbolic spaces are smooth Riemannian manifolds $\mathcal{M} = \H^d$ and as such locally Euclidean spaces. In the following we introduce basic notation for three popular models of hyperbolic spaces. For a comprehensive overview see~\citet{bridson}.\\

\subsection{Models of hyperbolic spaces}
\label{sec:A1}
\begin{figure}[ht]
    \centering
    \includegraphics[scale=0.1]{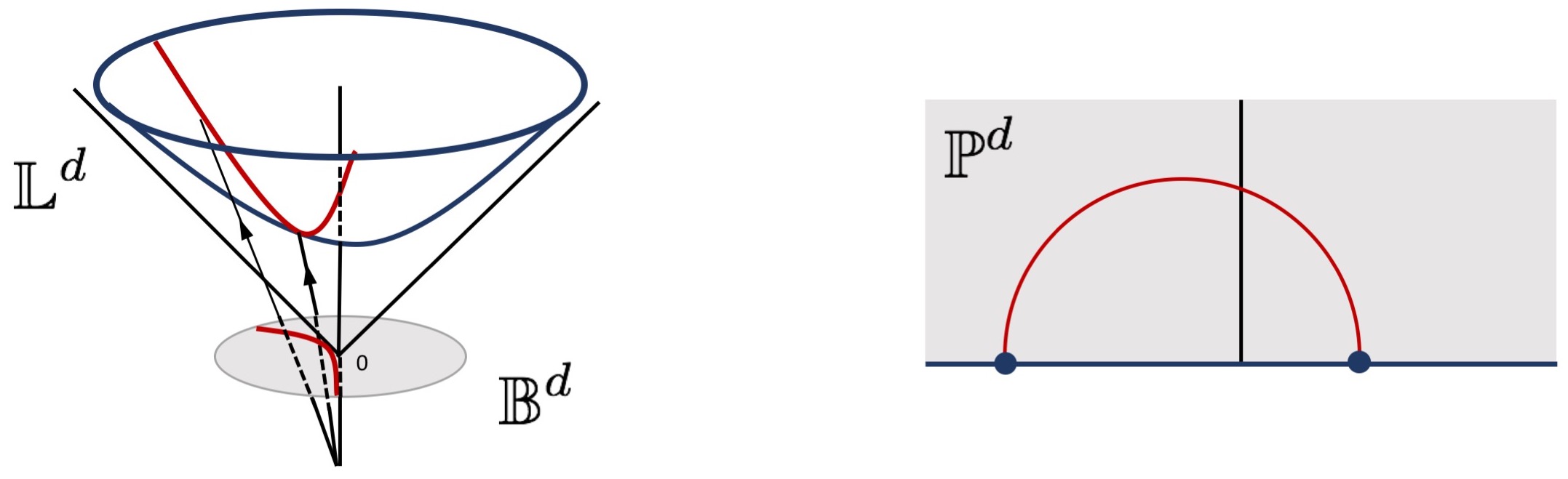}
    \caption{Models of hyperbolic space: The Lorentz model $\L^d$, the Poincare ball $\B^d$, and the Poincare half-plane $\Pb^d$.}
    \label{fig:models}
\end{figure}
\noindent The Poincare ball defines a hyperbolic space within the Euclidean unit ball, i.e.
\begin{align*}
\B^d &= \lbrace \vx \in \R^d : \; \norm{\vx} <1	\rbrace \\
d_{\B} (\vx,\vx') &= \acosh \left( 1 + 2 \frac{\norm{\vx-\vx'}^2}{(1-\norm{\vx}^2)(1-\norm{\vx'}^2)} \right) \; .
\end{align*}
Here, $\norm{\cdot}$ is the usual Euclidean norm.

The closely related Poincare half-plane model is defined as
\begin{align*}
\Pb^2 &= \lbrace \vx \in \R^2  \; \colon \;  x_1 >0  \rbrace \\
d_{\Pb} (\vx, \vx') &= \acosh \left( 1 + \frac{(x'_0 - x_0)^2 + (x'_1 - x_1)^2}{2 x_1 x'_1}  \right) \; .
\end{align*}
Note that if $x_0 = x'_0$, the metric simplifies as
\begin{align*}
d_{\Pb} (\vx, \vx') = d_{\Pb}((x_0,x_1), (x_0,x'_1)) = \Big\vert \ln \frac{x'_1}{x_1} \Big\vert \; .
\end{align*}
The model can be generalized to higher dimensions with
\begin{align*}
    \Pb^d = \lbrace (x_0, \dots, x_{d-1}) \in \R^d  \; \vert \;  x_{d-1} >0  \rbrace \; ,
\end{align*}
however, we will only use the two-dimensional model $\Pb^2$ here. We further define the hyperboloid as
\begin{align*}
\L^d &= \lbrace \vx \in \R^{d+1} : \; \vx*\vx = 1 \rbrace \\
d_{\L} (\vx, \vx') &= \acosh(\vx*\vx') \; ,
\end{align*}
where $\ast$ denotes the Minkowski product $\vx*\vx'=x_0 x'_0 - \sum_{i=1}^d x_i x'_i$. 
\begin{rmk}\label{rmk:double-sheet}
\normalfont
The Lorentz model
\begin{align*}
    \L^d = \lbrace x \in \R^{d+1}: \; \vx * \vx = 1 \rbrace \; .
\end{align*}
is also called double-sheet model.
We use this more general setting in sections 2-4. For simplicity, we restrict ourselves to the upper sheet
\begin{align*}
    \L_{+}^d = \lbrace \vx \in \R^{d+1}: \; \vx * \vx = 1, \; x_0 >0 \rbrace \; ,
\end{align*}
in section 5. All constructions of mappings between the different models of hyperbolic space can be extended to the double-sheet $\L^d$.
\end{rmk}


\subsection{Equivalence of different models of hyperbolic spaces}\label{sec:equiv}
The Poincare ball $\B^d$ and the Lorentz model $\L_{+}^d$ are equivalent models of hyperbolic space. A mapping is given by
\begin{align*}
    \pi_{\rm LB} : \L_{+}^d &\rightarrow \B^d \\
    \vx = (x_0, \dots, x_d) &\mapsto \left( \frac{x_1}{1+x_0}, \dots, \frac{x_d}{1+x_0} \right) \; .
\end{align*}

We can further construct a mapping from $\B^d$ to $\Pb^d$ by inversion on a circle centered at $(-1,0,\ldots,0)$:
\begin{align*}
\pi_{\rm BP}: \B^d &\rightarrow \Pb^d \\
\vx = (x_0, \dots, x_{d-1}) &\mapsto \frac{(2x_1,\dots,2x_{d-1},1-\norm{\vx}^2)}{1+2x_0 + \norm{\vx}^2} \; .
\end{align*}

\subsection{Embeddability}
\label{sec:embed}
When analyzing the dimension-distortion trade-off, we make use of two key results on the embeddability (cf.~\S\ref{subsec: embed}) of trees into Euclidean and hyperbolic spaces. We state them below for reference.
\begin{theorem}[\citep{bourgain}]\label{thm:bourgain}
An $N$-point metric $\Xc$ (i.e., $\vert \Xc \vert = N$) embeds into Euclidean space $\R^{O \left(\log^2 N\right)}$ with the distortion $c_M=O(\log N)$.
\end{theorem}
This bound in Theorem~\ref{thm:bourgain} is tight for trees in the sense that embedding them in a Euclidean space (of any dimension) must incur the distortion $c_m = \Omega(\log N)$~\citep{linial1995geometry}.
\begin{theorem}[\citep{sarkar}]\label{thm:sarkar}
Tree metrics embed quasi-isometrically with $c_M = O(1+\epsilon)$ into $\H^d$. 
\end{theorem}
%

\subsection{Spherical codes in hyperbolic space}
\label{sec:A3}
\noindent Consider the unit sphere $\mathbb{S}^{d-1} \subseteq \R^d$. A \emph{spherical code} is a subset of $\mathbb{S}^{d-1}$, such  that any two distinct elements $\vx,\vx'$ are separated by at least an angle $\theta$, i.e. $\ip{\vx}{\vx'} \leq \cos \theta$. We denote the size of the largest code as $A(d, \theta)$.

A similar construction of such ``spherical caps" can be obtained in $\H^d$. Note that the induced geometry of these caps is spherical, hence they inherit a spherical geometric structure. This allows in particular the transfer of bounds on $A(d,\theta)$ to hyperbolic space~\citep{cohn}:
\begin{theorem}[Chabauty, Shannon, Wyner (see, e.g.,~\citep{shannon-bound})]\label{thm:sh-bound}
$A(d,\theta) \geq (1+o(1))\sqrt{2\pi d} \frac{\cos \theta}{\sin^{d-1}\theta}$.
\end{theorem}

\section{Adversarial Learning}
\subsection{Loss functions}
\label{sec:C1}
For training the classifier, we consider the margin losses that have the following form
\begin{align}
\label{eq:f_def_appen}
l( \vx, y; \vw ) = f( y \cdot (\vw * \vx) ),
\end{align}
\noindent where $f \colon \mathbb{R} \to \mathbb{R}_+$ is some convex, non-increasing function. \citet{cho} introduce the \emph{hinge loss} in the hyperbolic setting which is defined by the (hyperbolic) hinge function $f(s) = \max\{0, \asinh(1) - \asinh(s)\}$, i.e.,
\begin{align}
\label{eq:cho-hinge}
l( \vx, y; \vw ) = \max \lbrace	0, \asinh(1) - \asinh(y (\vw*\vx))	\rbrace \; .
\end{align}
\noindent A significant shortcoming of this notion is its non-smoothness and non-convexity. Therefore, we additionally consider a smoothed \emph{least squares loss}:
\begin{align}\label{eq:square-loss}
 l(\vx_i, y_i; \vw) &= 
 \begin{cases}
 \frac{1}{2} \left(	\asinh(1) - \asinh(y_i (\vw*\vx_i)	\right)^2, & y_i (\vw*\vx_i) \leq 1\\
 0, &{\rm else}
 \end{cases} \; ,
\end{align}
\noindent We present experimental results for both losses.\\

\noindent The majority of the paper employs a hyperbolic version of the \emph{logistic loss} to introduce the logistic regression problem in hyperbolic space.  First, recall the logistic regression problem in the Euclidean setting. Given an input $\vx$ and a linear classifier defined by $\vw$, the prediction of the classifier is defined as
\begin{align}
p (y \vert \vx; \vw) &= 1/\big( 1 + \exp( - y\ip{\vx}{\vw})\big) \;
\end{align}
Thus the logistic loss takes the following form
\begin{align}
\label{eq:euclidean-log-loss}
l(\vx, y; \vw) = - \log p (y \vert \vx; \vw) &= \log \big( 1 + \exp( - y\ip{\vx}{\vw})\big) \; ,
\end{align}
\noindent We can define a hyperbolic version of the logistic regression problem, where we replace the Euclidean inner product with the Minkowski product, i.e.,
\begin{align}\label{eq:hyp-log-loss}
l(\vx,y; \vw) &= \log (1+ \exp(-y(\vx * \vw))) \; .
\end{align}

\noindent The following result verifies that the loss in \eqref{eq:hyp-log-loss} indeed satisfies Assumption~\ref{ass1}.
\begin{lem}
For valid inputs $(\vx,y;\vw)$, the hyperbolic logistic loss in \eqref{eq:hyp-log-loss} fulfills Assumption~\ref{ass1}.
\end{lem}
\begin{proof}
Note that the hyperbolic logistic loss reduces to its Euclidean counterpart, with the transformation $\hat{x}=(x_0,-x_1,\dots,x_d)$, since $\vx * \vw = \ip{\hat{x}}{w}$. It is easy to verify that the Euclidean logistic loss fulfills the (standard) assumptions in~~\ref{ass1}(iii).
\end{proof}

\subsection{Generating adversarial examples (Certification problem)}
\label{sec:C2}

\noindent Recall that to train a classifier with large margin, we enrich the training set with adversarial examples~(cf.~Algorithm~\ref{alg:adv-mf}). For a classifier $\vw$, an adversarial example $\tilde{\vx}$ for a given $(\vx, y)$ is generated by perturbing $\vx$ in the hyperbolic space up to the maximum allowed perturbation budget $\alpha$ such that 
\begin{align*}
\tilde{\vx} \gets \argmax_{\substack{\vz \in \L^d \\ d_\L(\vx,\vz) \leq \alpha}} l(\vz,y;\vw) \; .
\end{align*}
For the underlying loss function (cf.~Section~\ref{sec:C1}), due to the monotonicity of $\asinh$, the above problem can be equivalently expressed as
\begin{align}
\tilde{\vx} &\gets \argmin_{\substack{\vz \in \L^d \\ d_\L(\vx,\vz) \leq \alpha}} y\cdot (\vw \ast \vz) = \argmax_{\substack{\vz \in \L^d \\ d_\L(\vx,\vz) \leq \alpha}} -\vw' \ast \vz \nonumber \\
&= \argmax_{\substack{\vz \in \L^d \\ d_\L(\vx,\vz) \leq \alpha}} -w'_0 z_0 + \sum_i w'_i z_i \label{eq:adv_formulation_appen}
\end{align}
where $\vw' = -y\vw$. Since $\vw', \vz \in \R^{d+1}$, we can rewrite \eqref{eq:adv_formulation_appen} as a constraint optimization task in the ambient Euclidean space:
\begin{align}
\label{eq:cert_pre_appen}
\max_{\vz \in \R^{d+1}} \; & -w_0 z_0 + \sum_i w_i z_i \\
{\rm s.t.} \quad \; &d_\L (\vx,\vz) \leq \alpha \nonumber \\
&z_0^2 - \sum_{i = 1}^{d} z_i^2 =1 
 \; . \nonumber
\end{align}
Assuming that we guess $z_0$ based on $x_0$, the constraint $z_0^2 - \sum_{i = 1}^{d} z_i^2 =1$ confines the solution space onto a $d$-dimensional sphere of radius $r=\sqrt{z_0^2-1}$, which also implies that $z_0 \geq 1$. On the other hand the constraint $d_\L (\vx,\vz) \leq \alpha$ is equivalent to
\begin{align*}
d_\L (\vx, \vz) &= \acosh(\vx*\vz) = \acosh(x_0 z_0 - \sum_{i=1}^{d} x_i z_i) < \alpha\quad \text{or}\quad\sum_i -x_i z_i \leq \cosh(\alpha) - x_0 z_0 \; .
\end{align*}
Thus, the problem in \eqref{eq:cert_pre_appen} reduces to the following linear program with a spherical constraint.
\begin{align}
\label{eq:cert_appen}
({\rm CERT}) \quad \max_{\vz_{\backslash 0} \in \R^d} \; & -w_0 z_0 + \sum_i w_i z_i \\
\quad {\rm s.t.} \quad \;  &\sum_{i=1}^{d} -x_i z_i \leq \cosh(\alpha) - x_0 z_0 \nonumber \\
\quad &\norm{\vz_{\backslash 0}}^2 = z_0^2 - 1 \;, \nonumber 
\end{align}
where $\vz_{\backslash 0} = (z_{1},\ldots, z_{d})$. We now present a proof of Theorem~\ref{thm:cert} which characterizes a solution of the program in \eqref{eq:cert_appen}. For the sake of readability, we first restate the result from the main text:
\begin{theorem}[Theorem~\ref{thm:cert}]
\label{thm:cert-appen}
Given the input example $(\vx, y)$, let $\vx_{\backslash 0} = (x_1,\ldots, x_d)$. We can efficiently compute a solution to ({\rm CERT}) or decide that no solution exists. If a solution exists, then based on a guess of $z_0$ a maximizing adversarial example has the form $\tilde{\vx} = \left(z_0, \sqrt{z_0^2 -1} \left(b \check{\vx} + \sqrt{1-b^2} \check{\vx}^{\perp} \right) \right)$. Here, $b= \frac{(\cosh(\alpha) - x_0 z_0 )}{(\norm{\vx_{\backslash 0}} \sqrt{z^2_0-1})}$ depends on the adversarial budget $\alpha$, and  $\check{\vx}^{\perp}_{\backslash 0}$ is a unit vector orthogonal to $\check{\vx} = -{\vx_{\backslash 0}}/{\norm{\vx_{\backslash 0}}}$ along $\vw$.
\end{theorem}
\begin{proof}
First, note that ({\rm CERT}) can be rewritten as
\begin{align*}
({\rm \check{CERT}}) \quad \max \; &\ip{\check{\vw}}{\check{\vz}} \\
\quad s.t. \;  &\ip{\check{\vx}}{\check{\vz}} \leq b \\
\quad &\norm{\check{\vz}} = 1 \; ,
\end{align*}
where $\check{\vw}= {\vw_{\backslash 0}}/{\norm{\vw_{\backslash 0}}}$, $\check{\vx} = -{\vx_{\backslash 0}}/{\norm{\vx_{\backslash 0}}}$, and $b= {(\cosh(\alpha) - x_0 z_0 )}/{(\norm{\vx_{\backslash 0}} \norm{\vz_{\backslash 0}})}$. We further set $\check{\vz}={\vz_{\backslash 0}}/{\norm{\vz_{\backslash 0}}}$ so that the norm constraint confines the solution to the unit sphere to simplify the derivation. We can later rescale the solution to have the norm $\sqrt{z_0^2 -1}$. 

The solution of $\check{\rm CERT}$ lies on the cone $\ip{\check{\vx}}{\check{\vz}}=b$. We decompose $\check{\vw}$ along $\check{\vx}$ and its orthogonal complement $\check{\vx}^{\perp}$, i.e.
\begin{align*}
\check{\vw} = \xi \check{\vx} + \zeta \check{\vx}^{\perp} \; .
\end{align*} 
with $\zeta\geq 0$ and $\norm{\check{\vx}^{\perp}} = 1$. Without loss of generality, such a decomposition always exists. Note that
\begin{align*}
\ip{\check{\vw}}{\check{\vz}^{\ast}} = \xi \ip{\check{\vx}}{\check{\vz}^{\ast}} + \zeta \ip{\check{\vx}^{\perp}}{\check{\vz}^{\ast}} = \xi b + \zeta \ip{\check{\vx}^{\perp}}{\check{\vz}^{\ast}} \; ,
\end{align*}
where the second equality follows from $\ip{\check{\vx}}{\check{\vz}^{\ast}}=b$. This implies that for the objective $\ip{\check{\vw}}{\check{\vz}}$ to be maximized, $\check{\vz}^{\ast}$ has to have all of its remaining mass along $\check{\vx}^{\perp}$, i.e.,
\begin{align*}
\check{\vz}^{\ast} = 
b \check{\vx} + \sqrt{1-b^2} \check{\vx}^{\perp}.
\end{align*}
After rescaling to satisfy the original norm constraint in {\rm CERT}, the maximizing adversarial example (for a given $z_0$) is given as $$\tilde{\vx} = \left(z_0, \sqrt{z^2_0 -1}\cdot \check{\vz}^{\ast}\right) =  \left(z_0, \sqrt{z_0^2-1} \left( b \check{\vx} + \sqrt{1- b^2} \check{\vx}^{\perp} \right) \right).$$ 
\end{proof}

\subsection{Adversarial Perceptron}
\label{sec:C4}
For the convergence analysis of the gradient-based update, we first need to analyze the convergence of the adversarial perceptron. We first state the following lemma that relates the adversarial margin to the max-margin classifier.
\begin{lem}
\label{lem:hyp-adv-margin}
Let $\bar{\vw} \in \mathbb{R}^{d+1}$ with $\sqrt{-\bar{\vw}*\bar{\vw}}=1$ be a max-margin classifier of $\mathcal{S}$ with margin $\gamma_H$. At each iteration of Algorithm~\ref{alg:adv-mf}, $\bar{\vw}$ linearly separates $\mathcal{S} \cup \mathcal{S}'$ with margin at least $\gamma_H-\alpha$.
\end{lem}
\begin{rmk}
Note that this ``adversarial Perceptron" corresponds to a gradient update of the form $\vw_{t+1} \gets \vw_t + y \widehat{\tilde{\vx}}$, which resembles an adversarial SGD algorithm.
\end{rmk}
\begin{proof}
For the proof we first recall some standard identities for hyperbolic functions that we will use throughout the proof ($a,b \in \mathbb{R}$):
\begin{enumerate}
\item $\sinh (a) + \sinh(b) = 2 \sinh \left(\frac{a+b}{2}	\right) \cosh \left( \frac{a-b}{2}\right)$
\item $\sinh(-a) = - \sinh(a)$
\item $\sinh \left(\frac{a}{2}	\right) = \frac{\sinh(a)}{\sqrt{2 (\cosh(a) + 1)}} = {\rm sgn}(a) \sqrt{\frac{1}{2} \left(	\cosh(a)-1	\right)}$
\end{enumerate}
For an adversarial example $\tilde{\vx} \in \mathcal{S} \cup \mathcal{S}'$, the max-margin classifier $\bar{\vw}$ achieves the following hyperbolic margin:
\begin{align*}
\tilde{\gamma}_H &= \asinh (\bar{\vw} * \tilde{\vx}) \\
&= \asinh (\bar{\vw} * (\tilde{\vx} - \vx + \vx)) \\
&= \asinh (\bar{\vw} * (\tilde{\vx} - \vx) + \bar{\vw} * \vx ) \\
&\overset{(1)}{\geq} \asinh (\underbrace{\bar{\vw} * \vx}_{\geq \sinh(\gamma_H)}
- \underbrace{\vert \bar{\vw} \vert}_{=1} \cdot \vert \tilde{\vx} - \vx \vert) \\
&\overset{(2)}{\geq} \asinh(\sinh(\gamma_H) -  \vert \tilde{\vx} - \vx \vert) \; ,
\end{align*}
where (1) follows from the Cuachy-Schwarz inequality for Minkowski products and (2) from the assumptions on $\bar{\vw}$.  We further have
\begin{align*}
\vert \tilde{\vx} - \vx \vert 
&= \sqrt{- (\tilde{\vx} - \vx) * (\tilde{\vx} - \vx)} \\
&= \sqrt{- (\underbrace{\tilde{\vx}* \tilde{\vx}}_{=1} - 2 \tilde{\vx} * \vx + \underbrace{\vx * \vx}_{=1})} \\
&= \sqrt{2(\tilde{\vx}* \vx) - 2} \\
&\overset{(3)}{\leq} \sqrt{2 \cosh(\alpha) - 2} \\
&\overset{(4)}{=} 2 {\rm sgn}(\alpha) \sqrt{\frac{1}{2} \left(	\cosh(\alpha) - 1	\right)} \\
&\overset{(5)}{=} 2 \sinh \left(\frac{\alpha}{2}	\right)\\
&\leq \sinh(\alpha)
\; .
\end{align*}
Here, (3) follows from the adversarial budget, i.e., from $\alpha \geq \acosh(\vx * \tilde{\vx})$; (4) follows from $\alpha >0$ by construction and (5) from standard relation 3 above. Note that $\asinh(a)$ is monotonically increasing for $a \geq 0$.  As a consequence,  plugging $\vert \tilde{\vx} - \vx \vert \leq \sinh(\alpha)$
into the inequality above gives
\begin{align*}
\tilde{\gamma}_H \geq \asinh(\sinh(\gamma_H) - \sinh(\alpha)) \; .
\end{align*}
Next, we analyze $(\sinh(\gamma_H) - \sinh(\alpha))$:
\begin{align*}
\sinh(\gamma_H) - \sinh(\alpha) 
&\overset{(6)}{=} \sinh(\gamma_H) + \sinh(- \alpha) \\
&\overset{(7)}{=} 2 \sinh \left( \frac{\gamma_H - \alpha}{2}	\right) \cosh \left(	\frac{\gamma_H + \alpha}{2}	\right) \\
&\overset{(8)}{=} \sinh(\gamma_H - \alpha) \underbrace{\frac{\sqrt{2} \cosh \left(	\frac{\gamma_H + \alpha}{2}\right)}{\cosh(\gamma_H - \alpha) + 1}}_{\geq 1} \\
&\geq \sinh(\gamma_H - \alpha) \; ,
\end{align*}
where (6) follows from standard relation 2,  (7) from standard relation 1 and (8) from standard relation 3. Inserting this above gives the claim as
\begin{align*}
\tilde{\gamma_H} \geq \gamma_H - \alpha \; .
\end{align*}
\end{proof}
\noindent With this result, we can show the following bound on the sample complexity of the adversarial perceptron:

\begin{theorem}
 Assume that there is some $\bar{\vw} \in \R^{d+1}$ with $\sqrt{-\bar{\vw}*\bar{\vw}}=1$, and some $\gamma_H>0$, such that $y_j (\bar{\vw}*\vx_j) \geq \sinh(\gamma_H)$ for $j=1, \dots, \vert \mathcal{S} \vert$. Then,  the adversarial perceptron (with adversarial budget $\alpha$) converges after $O\left(\frac{1}{\sinh(\gamma_H - \alpha)}\right)$ steps, at which it has margin of at least $\gamma_H-\alpha$.
\end{theorem}
\begin{proof}
The proof adapts the proof technique of Theorem~\ref{prop:h-perceptron} to the adversarial setting. Assume $\vw_0 = \bm{0}$.
Furthermore, assume that the $t$th error is made at the $j${th} sample. 
Thus, 
\begin{align*}
    \vw_{t+1} \gets \vw_t + y_j \widehat{\tilde{\vx}}_j \; ,
\end{align*}
which implies that
\begin{align*}
    (\vw_{t+1} - \vw_t)^T \bar{\vw} = \left( y_j \tilde{\vx}_j \right)^T \bar{\vw} = y_j \left( \tilde{\vx}_j * \bar{\vw}  \right) 
    \geq \sinh(\gamma_H-\alpha) \; ,
\end{align*}
where the last inequality follows from Lemma~\ref{lem:hyp-adv-margin}. By summing and telescoping, we obtain that
\begin{align*}
    \sum_{k=0}^t (\vw_{k+1}-\vw_k)^T \bar{\vw} &\geq \sum_{k=0}^t \sinh(\gamma_H - \alpha) \\
    \Rightarrow  \quad (\vw_{t+1}-\vw_0)^T \bar{\vw} &\geq t \sinh(\gamma_H-\alpha) \; .
\end{align*}
Furthermore, note that
\begin{align*}
\norm{\vw_{t+1}}^2 &= \norm{\vw_t + y_j \widehat{\tilde{\vx}}_j}^2 \\
&= \ip{\vw_t}{\vw_t} + 2 y_j \ip{\vw_t}{\widehat{\tilde{\vx}}_j} + y_j^2 \ip{\widehat{\tilde{\vx}}_j}{\widehat{\tilde{\vx}}_j} \\
&= \norm{\vw_t}^2 + \underbrace{2 y_j (\vw_t * \tilde{\vx}_j)}_{\leq 0} + \underbrace{\ip{\tilde{\vx}_j}{ \tilde{\vx}_j}}_{\leq R_x} \\
&\overset{(1)}{\leq}  \norm{\vw_t}^2 + R_x^2 \; ,
\end{align*}
where (1) follows from Assumption~\ref{ass1}(2).  Recursively, this implies $\norm{\vw_{t+1}}^2 \leq t R_x^2$.  
Now,  note that for all $t \geq 0$
\begin{align*}
1 \geq \frac{\vw_{t+1}^T \bar{\vw}}{\norm{\vw_{t+1}} \norm{\bar{\vw}} } 
\geq \sqrt{t} \frac{\sinh(\gamma_H - \alpha)}{R_x \norm{\bar{\vw}}} \; .
\end{align*}
This implies
\begin{align*}
    t \leq \left( \frac{R_x \norm{\bar{\vw}}}{\sinh(\gamma_H - \alpha)} \right)^2 \; .
\end{align*}
\end{proof}

\subsection{Gradient-based update}
\label{appen:C5}
Recall that our objective in Algorithm~\ref{alg:adv-mf} consists of an inner optimization (that computes the adversarial example) and an outer optimization (that updates the classifier). In particular, we consider
\begin{align*}
\min_{\vw \in \R^{d+1}} \; L_{\rm rob}(\vw; \Sc) := \frac{1}{\vert \Sc \vert} \sum_{(x,y) \in \Sc} l_{\rm rob} (\vx,y; \vw) \; ,
\end{align*}
where the robust loss is given by
\begin{align*}
l_{\rm rob} (\vx,y; \vw) := \max_{\vz \in \L^d, d_{\L}(\vx,\vz) \leq \alpha} l(\vx,y; \vw) = l(\tilde{\vx}, y; \vw)\; ,
\end{align*}
where $\tilde{\vx} \in \argmax_{\vz \in \L^d, d_{\L}(\vx,\vz) \leq \alpha} l(\vx,y; \vw)$.

Recall that, to compute the update, we need to compute gradients of the outer minimization problem, i.e., $\nabla_{w} \; l_{\rm rob}$ over $\mathcal{S}$. However, the function $l_{\rm rob}$ is itself a maximization problem (referred to as the inner maximization problem above). Therefore, we compute the gradient at the maximizer of the inner problem. Danskin's theorem ensures that this gives a valid descent direction. For the sake of completeness, we recall the Danskin's theorem here.
\begin{theorem}[~\citet{danskin,bertsekas}]
Suppose $X$ is a non-empty compact topological space and $g: \R^d \times X \rightarrow \R$ is a continuous function such that $g(\cdot, \delta)$ is differentiable for every $\delta \in X$. Let $\delta^*_{\vw} = \argmax_{\delta \in X} g(\vw,\delta)$. Then, the function $\psi(\vw) = \max_{\delta \in X} g(\vw, \delta)$ is subdifferentiable and the subdifferential is given by
\begin{align*}
\partial \psi (\vw) = {\rm conv} \left( \lbrace \nabla_{\vw} \; g(\vw, \delta) \vert \; \delta \in \delta_{\vw}^* \rbrace \right) \; .
\end{align*}
\end{theorem}
\noindent This approach has been previously used in~\citet{madry} and~\citet{charles}. Note that when we find an adversarial example in Algorithm~\ref{alg:adv-mf}, we can write it in a closed form (cf.~Theorem~\ref{thm:cert-appen}). In particular,
\begin{align*} 
l_{\rm rob}(\vx,y; \vw) =\max_{d_\L (\vx, \vz) \leq \alpha } l(\vz, y ; \vw) = l(\tilde{\vx},y;\vw) \quad {\rm with} \; \tilde{\vx} = \left(\tilde{x}_0, \sqrt{\tilde{x}_0^2-1} \left( b \
\check{\vx} + \sqrt{1- b^2} \check{\vx}^{\perp} \right) \right) \; .
\end{align*}
Note that
\begin{align*}
\nabla_{\vw} \; l(\tilde{\vx},y;\vw) = f'(y(\vw*\tilde{\vx})) \cdot \nabla_{\vw} \; y(\vw*\tilde{\vx}) \; = f'(y(\vw*\tilde{\vx})) \cdot y \widehat{(\tilde{\vx})}^T \; ,
\end{align*}
where we have used the fact that $\nabla_{\vw} \; y(\vw*\tilde{\vx}) = y\widehat{(\tilde{\vx})}^T = y(\tilde{x}_0, -\tilde{x}_1, \dots, -\tilde{x}_n)^T$. From Danskin's theorem, we have $\nabla_{\vw} l(\tilde{\vx},y;\vw) \in \partial \; l_{\rm rob}(\vx,y; \vw)$. This enables us to compute the descent direction and perform the update step with
\begin{align*}
\nabla \; L(\vw; \Sc') = \frac{1}{\vert \mathcal{S}' \vert} \sum_{(\tilde{\vx},y) \in \mathcal{S}'} \nabla\; l(\tilde{\vx},y;\vw)  \; \in \partial L_{\rm rob}(\vw; \Sc) \; ,
\end{align*}
Furthermore, we have
\begin{align}
\label{eq:lrob-hessian}
\nabla_{\vw}^2 \; l(\tilde{\vx},y;\vw) = f''(y(\vw*\tilde{\vx})) \tilde{\vx} \tilde{\vx}^T \in \partial^2 l_{\rm rob}(\vx, y; \vw) \; ,
\end{align}
which enable the computation of the Hessian of $L(\vw; \Sc')$.

\noindent The convergence results in this section build on hyperbolic analogues of comparable Euclidean results in~\citep{srebo,telgarsky}. 

\noindent We first show a bound on the Hessian of the loss:
\begin{lem}\label{prop:beta}
\begin{align*}
 \nabla^2 L(\vw_t; \Sc'_t)\preceq \beta \sigma_{\rm max}^2\cdot I \;,
\end{align*}
 where $\sigma_{\rm max}$ is an upper bound on the maximum singular value of the data matrix 
 $
 \frac{1}{\vert \Sc'_t \vert}\sum_{(\tilde{\vx}, y) \in \Sc'_t} \tilde{\vx}\tilde{\vx}^T.
 $
\end{lem}
\begin{proof}
\begin{align*}
    \nabla^2 L(\vw_t; \Sc'_t) &= \frac{1}{\vert \Sc'_t \vert} \sum_{(\tilde{\vx},y) \in \Sc'_t} \nabla^2 l(\tilde{\vx},y; \vw_t) \overset{(i)}{=} \frac{1}{\vert \Sc'_t \vert} \sum_{(\tilde{\vx},y) \in \Sc'_t} f''(y(\tilde{\vx}*\vw_t)) \tilde{\vx} \tilde{\vx}^T \\
    &\overset{(ii)}{\preceq} \beta\cdot \frac{1}{\vert \Sc'_t \vert} \sum_{(\tilde{\vx},y) \in \Sc'_t} \tilde{\vx} \tilde{\vx}^T \preceq \beta \sigma_{\rm max}^2\cdot I \; ,
\end{align*}
where $(i)$ and $(ii)$ follow from \eqref{eq:lrob-hessian} and the assumption that $f$ is $\beta$-smooth.
\end{proof}
\noindent With the help of Lemma~\ref{prop:beta}, we can show the following result (a restatement of Theorem~\ref{conv:alg-2-main}), which establishes that the gradient updates are guaranteed to converge to a large-margin classifier:
\begin{theorem}[Theorem~\ref{thm:gd-conv}]
Let $\lbrace \vw_t \rbrace_{t \geq 0}$ be the GD iterates 
\begin{align*}
    \vw' &\gets \vw_t - \frac{\eta}{\vert \mathcal{S}_t' \vert} \sum_{(\tilde{\vx},y) \in \mathcal{S}_t'} \nabla l(\tilde{\vx}, y;\vw) \\
    \vw_{t+1} &\gets \frac{\vw'}{\sqrt{- \vw' \ast \vw'}}
\end{align*}
with constant step size $\eta < \frac{2}{\beta \sigma_{\rm max}^2}$ and an initialization $\vw_0$ with $\vw_0 * \vw_0 <0$. Then, we have $\lim_{t \rightarrow \infty} L_{\rm rob} (\vw_t; \Sc) = \lim_{t \rightarrow \infty} L(\vw_t; \Sc_t' ) = 0$.
\end{theorem}
\begin{proof}
The proof adopts a result for a Euclidean adversarial gradient descent algorithm in~\citep[Lemma 1]{srebo}. By Assumption~\ref{ass1}(i) we can find a $\bar{\vw}$ that linearly separates $\Sc$. Then, we have
\begin{align*}
    \ip{\bar{\vw}}{\nabla L(\vw; \Sc'_t)} 
    &= \ip{\bar{\vw}}{\frac{1}{\vert \mathcal{S}_t' \vert} \sum_{(\tilde{\vx},y) \in \mathcal{S}_t'} f'(y(\tilde{\vx}*\vw_t))\; y\widehat{\tilde{\vx}}} \\
    &=
\underbrace{\left( \frac{1}{\vert \mathcal{S}_t' \vert} \sum_{(\tilde{\vx},y) \in \mathcal{S}_t'} f'(y(\tilde{\vx}*\vw_t)) \right)}_{<0}
    \underbrace{y\ip{\bar{\vw}}{\widehat{\tilde{\vx}}}}_{= y(\bar{\vw}* \tilde{\vx}) > \gamma_H - \alpha > 0} \; ,
\end{align*}
where the negativity of the first term follows from the assumptions on $f$ (cf.~Assumption~\ref{ass1}(iii)) and the lower bound on the second term from the separability assumption and Lemma~\ref{lem:hyp-adv-margin}. This implies that $\ip{\bar{\vw}}{\nabla L(\vw; \Sc'_t)} \neq 0$ for any finite $\vw$. Therefore, there are no finite critical points $\vw$ for which $\nabla L(\vw; \Sc'_t)=0$. However,  gradient descent is guaranteed to converge to a critical point for smooth objectives with an appropriate step size. Therefore, $\norm{\vw_t} \rightarrow \infty$ and $y(\vw_t*\tilde{\vx}) >0 \; \forall \; (\tilde{\vx},y) \in \mathcal{S}_t'$ and large enough $t$. Then the monotonicity assumption (Assumption~\ref{ass1}(iii)) implies that $l(\tilde{\vx},y; \vw_t) \rightarrow 0$ for all $(\tilde{\vx},y) \in \Sc_t'$.  This further implies that $L(\vw_t; \Sc'_t)=\frac{1}{\vert \Sc_t' \vert} \sum_{(\tilde{\vx},y) \in \mathcal{S}_t'} l(\tilde{\vx}, y; \vw_t) \rightarrow 0$.
\end{proof}
We further show that the enrichment of the training set with adversarial examples is critical for polynomial-time convergence: Without adversarial training, we can construct a simple max-margin problem, that cannot be solved in polynomial time. 
\begin{theorem}[Theorem~\ref{prop:gd-exp}]
Consider $\mathcal{S}=\lbrace ({\ve}_1,1), (-{\ve}_1,-1) \rbrace \subset \mathbb{R}^{d+1} \times \{+1, -1\}$ and a typical initialization $\vw_0= \ve_2 \in \mathbb{R}^{d+1}$ (with the standard basis vectors $\ve_1, \ve_2 \in \R^{d+1}$). 
Let $\{\vw_t\}_t$ be a sequence of classifiers generated by the GD updates (with fixed step size $\eta=1$), i.e., 
\begin{align*}
    \vw' &\gets \vw_t - \frac{\eta}{\vert \mathcal{S} \vert} \sum_{(\vx,y) \in \mathcal{S}} \nabla l(\vx, y;\vw)  \\
    \vw_{t+1} &\gets \frac{\vw'}{\sqrt{- \vw' \ast \vw}} \; .
\end{align*}
Then, the number of iterations needed to achieve margin $\gamma_H$ is $\Omega(\exp(\gamma_H))$. 
\end{theorem}
\begin{proof}
First, note that the gradient of the loss can be computed as
\begin{align*}
 \nabla l(\vx_i, y_i ;\vw_t) &= f'(y_i(\vx_i*\vw_t)) y_i\widehat{\vx_i} 
\end{align*}
where 
\begin{align*}
f'(s) = -\frac{1}{\exp(s)+1}
\end{align*}
 is the derivative of the hyperbolic logistic regression loss (cf.~\eqref{eq:hyp-logistic-reg-loss}). Note that due to the structure of $\mathcal{S}$ and $\vw_0$, the GD update will produce the following iteration sequence
\begin{align*}
    a_{t+1} &= a_{t} - f'(a_t) \qquad (a_0 =0)\\
    \vw_t &= (a_t, \sqrt{a_t^2 + 1}, 0, \dots,0) \; , 
\end{align*}
where the $\vw_t$ are determined through the GD update. Note that $\vw_t \ast \vw_t = a_t^2 - \left( \sqrt{a_t^2 + 1} \right)^2 = -1 <0$.
We now want to show by induction that
\begin{align*}
    a_t \leq \sinh(\ln(t+1)) \; .
\end{align*}
For the base case, note that $a_0 =  \sinh(\ln(1)) =0$.
Assume, that $a_t \leq \sinh(\ln(t+1))$. We want to show that
\begin{align*}
    a_{t+1} \leq \sinh(\ln(t+2)) \;.
\end{align*} 
Note, that
\begin{align}\label{eq:a_t_upperbound}
a_{t+1} &= a_t + \frac{1}{\exp(a_t) + 1} \leq \sinh(\ln(t+1)) + \frac{1}{2} \leq \frac{1}{2}\left(t+2 - \frac{1}{t+2}\right) \\
&= \frac{1}{2} \left( e^{\ln(t+2)} - e^{-\ln(t+2)} \right) = \sinh \left( \ln(t+2) \right)
\; ,
\end{align}
where we have inserted the induction assumption for the first inequality and utilized that, by definition, $\sinh(z)=\frac{1}{2}\left( e^z - e^{-z} \right)$ for the second. 
This finishes the induction proof. Assuming that $\vw_t$ achieves a margin of at least $\gamma_H$, we have
\begin{align*}
    \gamma_H \leq {\rm margin}_{\mathcal{S}}(\vw_t)
    = \asinh \left( \frac{y(\vx*\vw_t)}{\sqrt{-\vw_t*\vw_t}}              \right) \overset{(1)}{\leq} \asinh(a_{t+1}) \overset{(2)}{\leq} \asinh(\sinh(\ln(t+2))) \leq \ln(t+2) \; ,
\end{align*}    
where (1) follows from $\vw_t =  (a_t, \sqrt{a_t^2+1}, 0, \dots, 0)$ and the monotonicity of $\asinh(\cdot)$ and (2) from the upper bound \ref{eq:a_t_upperbound}. Now, by solving for $t$, we obtain that $t = \Omega(\exp(\gamma_H))$.
\end{proof}
\noindent Next, we quantify the convergence rate of adversarial training with GD updates (cf.~\eqref{eq:adv-gd-1}). We start by presenting some auxiliary results.
\begin{lem}[Smoothness bound]\label{prop:smooth-bound}
 Let 
 $\eta_t =: \eta < \frac{2\sinh^2(\gamma_H - \alpha)}{\beta \sigma_{\rm max}^2 R_x^2}$
 be the fixed step size and $\vw_0$ a valid initialization, i.e. $\vw_0 * \vw_0 < 0$. Then,  with the GD update (with fixed step size $\eta_t = : \eta$)
\begin{align*}
    \vw' &\gets \vw_t - \eta_t \underbrace{\nabla L(\vw_t; \Sc'_t)}_{\in \partial L_{\rm rob}(\vw_t; \Sc_t)} \\
    \vw_{t+1} &\gets \frac{\vw'}{\sqrt{- \vw' \ast \vw'}}
    \; ,
\end{align*}
we have 
\begin{enumerate}
    \item $L_{\rm rob}(\vw_{t+1}; \Sc) \leq L_{\rm rob}(\vw_{t}; \Sc) - \eta \left( \frac{\sinh(\gamma_H - \alpha)^2}{R_x^2} - \frac{\beta \sigma_{\rm max}^2 \eta}{2}        \right) \norm{\underbrace{\nabla L(\vw_t; \Sc'_t)}_{\in \partial L_{\rm rob}(\vw_t; \Sc_t)}}^2$;
    \item $\sum_{k=0}^\infty \norm{\nabla L(\vw_k; \Sc'_k)}^2 < \infty$; as a result, $\lim_{t \rightarrow \infty} \norm{\nabla L(\vw_t); \Sc'_t}^2=0$.
\end{enumerate}
\end{lem}
\begin{proof}
In Algorithm~\ref{alg:adv-mf} with gradient update rule, we have
\begin{align*}
    \vw_{t+1} &= \vw_t - \eta \nabla L(\vw_t; \Sc'_t) \\
    &= \vw_t - \frac{\eta}{\vert \mathcal{S}_t' \vert} \sum_{(\tilde{\vx},y) \in \mathcal{S}_t'} l(\tilde{\vx},y;\vw_t) \\
    &= \vw_t - \frac{\eta}{\vert \mathcal{S}_t' \vert} \sum_{(\tilde{\vx},y) \in \mathcal{S}_t'} f'(y(\tilde{\vx}*\vw_t)) y \widehat{\tilde{\vx}} \; .
\end{align*}
Now, consider the inner product $\ip{\vw_{t+1}}{\bar{\vw}}$, where $\bar{\vw}$ is the optimal classifier. With out loss of generality, we assume $\norm{\bar{\vw}}=1$.
\begin{align*}
   \ip{\vw_{t+1}}{\bar{\vw}} &= \ip{\vw_{t}}{\bar{\vw}} - \frac{\eta}{\vert \mathcal{S}' \vert} \sum_{(\tilde{\vx},y) \in \mathcal{S}'} f'(y(\tilde{\vx}*\vw_t)) y \ip{\widehat{\tilde{\vx}}}{\bar{\vw}} \\
   &\overset{(i)}{=} \ip{\vw_{t}}{\bar{\vw}} - \frac{\eta}{\vert \mathcal{S}_t' \vert} \sum_{(\tilde{\vx},y) \in \mathcal{S}_t'} f'(y(\tilde{\vx}*\vw_t)) y (\tilde{\vx}*\bar{\vw}) \\
   &\overset{(ii)}{\geq} \ip{\vw_{t}}{\bar{\vw}} - \frac{\eta \sinh(\gamma_H - \alpha)}{\vert \mathcal{S}_t' \vert} \sum_{(\tilde{\vx},y) \in \mathcal{S}_t'} f'(y(\tilde{\vx}*\vw_t)) \; ,
\end{align*}
where $(i)$ and $(ii)$ follow from $\ip{\widehat{\tilde{\vx}}}{\bar{\vw}} = \tilde{\vx}*\bar{\vw}$ and $y(\tilde{\vx}*\bar{\vw}) \geq \sinh(\gamma_H - \alpha)$ (cf.~Lemma~\ref{lem:hyp-adv-margin}), respectively. 
With the linearity of the inner product, we get
\begin{align*}
    \ip{\vw_{t+1}-\vw_t}{\bar{\vw}} 
    &\geq  - \frac{\eta \sinh(\gamma_H - \alpha)}{\vert \mathcal{S}_t' \vert} \sum_{(\tilde{\vx},y) \in \mathcal{S}_t'} f'(y(\tilde{\vx}*\vw_t)) 
     \; .
\end{align*}
Since $f'$ is negative (cf.~Assumption~\ref{ass1}.3), we can replace $-f'(y(\tilde{\vx}*\vw_t))$ with $\vert f'(y(\tilde{\vx}*\vw_t))\vert$ to get
\begin{align}
\label{eq:pre-CS}
    \ip{\vw_{t+1}-\vw_{t}}{\bar{\vw}} &\geq \frac{\eta \sinh(\gamma_H - \alpha)}{\vert \mathcal{S}_t' \vert} \sum_{(\tilde{\vx},y) \in \mathcal{S}_t'} \vert f'(y(\tilde{\vx}*\vw_t)) \vert \nonumber \\
    &\overset{(iii)}{\geq} \frac{\eta \sinh(\gamma_H-\alpha)}{R_x} \norm{\nabla L(\vw_t; \Sc'_t)} \; ,
\end{align}
where $(iii)$ holds due to the following argument: Recall that $\norm{\nabla l(\tilde{\vx},y;\vw_t)} \leq \vert f'(y(\tilde{\vx}*\vw_t)) \vert \norm{\widehat{\tilde{\vx}}}$. Thus, \begin{align*}
    \norm{\nabla L(\vw_t; \Sc'_t)} &= \norm{\frac{1}{\vert \Sc_t' \vert} \sum_{(\tilde{\vx},y) \in \Sc_t'} l(\tilde{\vx},y;\vw_t)} \leq \frac{1}{\vert \Sc_t' \vert} \sum_{(\tilde{\vx},y) \in \Sc_t'} \norm{l(\tilde{\vx},y;\vw_t)} \\
    &\leq \frac{1}{\vert \Sc_t' \vert} \sum_{(\tilde{\vx},y) \in \Sc_t'} \vert f'(y(\tilde{\vx}*\vw_t)) \vert \norm{\hat{\tilde{\vx}}} \leq \frac{R_x}{\vert \Sc_t' \vert} \sum_{(\tilde{\vx},y) \in \Sc_t'} \vert f'(y(\tilde{\vx}*\vw_t)) \vert \; .
\end{align*}
This implies that
\begin{align*}
    \frac{1}{\vert \Sc_t' \vert} \sum_{(\tilde{\vx},y) \in \Sc_t'} \vert f'(y(\tilde{\vx}*\vw_t)) \vert \geq \frac{1}{R_x} \norm{\nabla L (\vw_t; \Sc'_t)} \; ,
\end{align*}
\noindent which implies (iii). Applying Cauchy-Schwarz to the left hand side of \eqref{eq:pre-CS} gives us that
\begin{align}
\label{eq:pos-CS}
    \norm{\vw_{t+1} - \vw_t} \; \norm{\bar{\vw}} \geq \ip{\vw_{t+1}-\vw_{t}}{\bar{\vw}} \geq 
    \frac{\eta \sinh(\gamma_H-\alpha)}{R_x} \norm{\nabla L(\vw_t; \Sc'_t)} \;.
\end{align}
Now, using the fact that $\norm{\bar{\vw}}=1$ in \eqref{eq:pos-CS}, we get
\begin{align}
\label{eq:pos-CS-1}
    \norm{\vw_{t+1} - \vw_t} \geq \frac{\eta \sinh(\gamma_H - \alpha)}{R_x} \norm{\nabla L(\vw_t; \Sc'_t)} \; .
\end{align}
Now, consider the following Taylor approximation:
\begin{align}
\label{eq:pre-psd}
    L_{\rm rob}(\vw_{t+1}; \Sc) &= L_{\rm rob}(\vw_t; \Sc) + \ip{\underbrace{\nabla L(\vw_t; \Sc'_t)}_{\in \partial L_{\rm rob}(\vw_t; \Sc)}}{\vw_{t+1} - \vw_t}\;  \nonumber \\
    &\qquad +{(\vw_{t+1}-\vw_t)^T \; \underbrace{\nabla^2 L(\vv; \Sc'_t)}_{\in \partial^2 L_{\rm rob}(\vv; \Sc)} \; (\vw_{t+1}- \vw_t)}/{2},
\end{align}
where $\vv \in {\rm conv}(\vw_{t+1},\vw_t)$. By utilizing Lemma~\ref{prop:beta} in \eqref{eq:pre-psd}, we get that
\begin{align}
\label{eq:post-psd}
    L_{\rm rob}(\vw_{t+1}; \Sc) &\leq L_{\rm rob}(\vw_t; \Sc) + \ip{\nabla L(\vw_t; \Sc'_t)}{\vw_{t+1} - \vw_t} + \frac{\beta \sigma_{\rm max}^2}{2} \norm{\vw_{t+1} - \vw_t}^2 \; .
\end{align}
\noindent Recall the update rule
\begin{align}
\label{eq:update-1}
    &\vw_{t+1} = \vw_t - \eta \nabla L(\vw_t; \Sc'_t) \\
    \Rightarrow~& \vw_{t+1} - \vw_t = - \eta \nabla L(\vw_t; \Sc'_t) \; .
\end{align}
Inserting this in \eqref{eq:post-psd}, we get
\begin{align}
\label{eq:post-psd-1}
     L_{\rm rob}(\vw_{t+1}; \Sc) &= L_{\rm rob}(\vw_t; \Sc) + \ip{-\eta^{-1}(\vw_{t+1} - \vw_t)}{\vw_{t+1} - \vw_t} + \frac{\beta \sigma_{\rm max}^2}{2} \norm{\vw_{t+1} - \vw_t}^2 \nonumber \\
    &= L_{\rm rob}(\vw_t; \Sc) - \eta^{-1} \norm{\vw_{t+1} - \vw_t}^2 + \frac{\beta \sigma_{\rm max}^2}{2} \norm{\vw_{t+1} - \vw_t}^2 \; .
\end{align}
By combining \eqref{eq:pos-CS-1} and \eqref{eq:post-psd-1}, we obtain that
\begin{align}
\label{eq:post-psd-2}
    L_{\rm rob}(\vw_{t+1}; \Sc) \leq L_{\rm rob}(\vw_t; \Sc) - \frac{\eta \sinh(\gamma_H - \alpha)^2}{R_x^2} \norm{\nabla L(\vw_t; \Sc'_t)}^2 + \frac{\beta \sigma_{\rm max}^2}{2} \norm{\vw_{t+1} - \vw_t}^2 \; .
\end{align}
Again, utilizing \eqref{eq:update-1}, it follows from \eqref{eq:post-psd-2} that
\begin{align}
\label{eq:post-psd-3}
     L_{\rm rob}(\vw_{t+1}; \Sc)  &\leq L_{\rm rob}(\vw_t; \Sc) - \frac{\eta \sinh(\gamma_H - \alpha)^2}{R_x^2} \norm{\nabla L(\vw_t; \Sc'_t)}^2 + \frac{\beta \sigma_{\rm max}^2 \eta^2}{2} \norm{\nabla L(\vw_t; \Sc'_t)}^2 \\
    &= L_{\rm rob}(\vw_t; \Sc) - \eta \left( \frac{ \sinh(\gamma_H - \alpha)^2}{R_x^2} - \frac{\beta \sigma_{\rm max}^2 \eta}{2} \right) \norm{\nabla L(\vw_t; \Sc'_t)}^2 \; .
\end{align}
This establishes the first claim of Lemma~\ref{prop:smooth-bound}. Now, we can rewrite \eqref{eq:post-psd-3} to obtain the following.
\begin{align*}
    \frac{L_{\rm rob}(\vw_t; \Sc) - L_{\rm rob}(\vw_{t+1}; \Sc)}{\eta \left( \frac{ \sinh(\gamma_H - \alpha)^2}{R_x^2} - \frac{\beta \sigma_{\rm max}^2 \eta}{2} \right)} \geq \norm{\nabla L(\vw_t; \Sc'_t)}^2 \; .
\end{align*}
Note that our assumption on the step size $\eta$ ensures that the denominator in \eqref{eq:post-psd-3} is $> 0$.
Next, summing and telescoping gives us that
\begin{align*}
    \sum_{k=0}^t \norm{\nabla L(\vw_k; \Sc'_k)}^2 
    \leq \sum_{k=0}^t \frac{L_{\rm rob}(\vw_k; \Sc) - L_{\rm rob}(\vw_{k+1}; \Sc)}{\eta \left( \frac{ \sinh(\gamma_H - \alpha)^2}{R_x^2} - \frac{\beta \sigma_{\rm max}^2 \eta}{2} \right)}
    = \frac{L_{\rm rob}(\vw_0; \Sc) - L_{\rm rob}(\vw_{t+1}; \Sc)}{\eta \left( \frac{\sinh( \gamma_H - \alpha)^2}{R_x^2} - \frac{\beta \sigma_{\rm max}^2 \eta}{2} \right)} \; ,
\end{align*}
where the right term is bounded, since $L_{\rm rob}(\vw_0; \Sc) < \infty$ and $0 \leq L_{\rm rob}(\vw_{t+1}; \Sc)$. This establishes the second claim of Lemma~\ref{prop:smooth-bound} as
\begin{align*}
    \sum_{k=0}^\infty \norm{\nabla L(\vw_k; \Sc'_k)}^2 < \infty  \quad \Rightarrow \quad
    \lim_{t \rightarrow \infty}\norm{\nabla L(\vw_t; \Sc'_t)}^2 =0 \; .
\end{align*}
\end{proof}
\begin{lem}\label{prop:tel-bound}
With the assumptions of Lemma~\ref{prop:smooth-bound}, Lemma~\ref{prop:smooth-bound}.1 implies for all $\vw \in\R^{d+1}$
\begin{align*}
    &2\sum_{k=0}^{t-1} \eta_k \big(L_{\rm rob}(\vw_k; \Sc) - L_{\rm rob}(\vw; \Sc)\big)
    + \sum_{k=0}^{t-1} \frac{\eta_k^2}{\bar{\eta}_k} \big(L_{\rm rob}(\vw_{k+1}; \Sc) - L_{\rm rob}(\vw_k; \Sc)\big) \nonumber \\
    &\qquad  \leq \norm{\vw_0 -\vw}^2 - \norm{\vw_t - \vw}^2\; ,
\end{align*}
where $\bar{\eta}_k = \eta_k \left( \frac{\sinh(\gamma_H - \alpha)^2}{ R_x^2}-\frac{\beta \sigma_{\rm max}^2 \eta_k}{4} \right)$.
\end{lem}
\begin{proof}
First, note that the GD update 
\begin{align*}
    \vw_{t+1} = \vw_t - \eta_t \underbrace{\nabla L(\vw_t; \Sc'_t)}_{\in \partial L_{\rm rob}(\vw_t; \Sc)}
\end{align*}
implies that
\begin{align}
\label{eq:norm-step1}
    \norm{\vw_{t+1} - \vw}^2 &= \norm{\vw_t - \vw}^2 - 2 \eta_t \ip{\nabla L(\vw_t; \Sc'_t)}{\vw_t - \vw} + \eta_t^2 \norm{\nabla L(\vw_t; \Sc'_t)}^2 \\
    &= \norm{\vw_t - \vw}^2 + 2 \eta_t \ip{\nabla L(\vw_t; \Sc'_t)}{\vw - \vw_t} + \eta_t^2 \norm{\nabla L(\vw_t; \Sc'_t)}^2  \; .
\end{align}
Note that the hyperbolic logistic loss $f(z)$~\eqref{eq:hyp-logistic-reg-loss} is convex. As a consequence, $l_{\rm rob}(\vx, y; \vw)$ is convex, i.e.,
\begin{align*}
    l_{\rm rob}(\vx, y; \vw) \geq l_{\rm rob}(\vx, y; \vw_t) + \ip{\partial l_{\rm rob}(\vx, y; \vw_t)}{\vw-\vw_t} \; ,
\end{align*}
for any $\vw \in \R^{d+1}$ and any pair $(\vx,y)$.
Since the sum of convex function is convex, we further have
\begin{align}
\label{eq:subgrad-full}
    L_{\rm rob}(\vw; \Sc) \geq L_{\rm rob}(\vw_t; \Sc) + \ip{\underbrace{\nabla L(\vw_t; \Sc'_t)}_{\in \partial L_{\rm rob}(\vw_t; \Sc)}}{\vw-\vw_t} \; .
\end{align}
By combining \eqref{eq:norm-step1} and \eqref{eq:subgrad-full}, we obtain that
\begin{align*}
    \norm{\vw_{t+1} - \vw}^2 
    &\leq \norm{\vw_t - \vw}^2 + 2 \eta_t \big( L_{\rm rob}(\vw; \Sc) - L_{\rm rob}(\vw_t; \Sc) \big) + \eta_t^2 \norm{\nabla L(\vw_t; \Sc'_t)}^2 \; \nonumber \\
    & \overset{(i)}{\leq} \norm{\vw_t - \vw}^2 + 2 \eta_t \big( L_{\rm rob}(\vw; \Sc) - L_{\rm rob}(\vw_t; \Sc) \big) + \frac{\eta_t ^2\big( L_{\rm rob}(\vw_t; \Sc) - L_{\rm rob}(\vw_{t+1}; \Sc)   \big)}{\bar{\eta}_t},  
\end{align*}
where $(i)$ follows from the first claim in Lemma~\ref{prop:smooth-bound} and $\bar{\eta}_t := \eta_t \left( \frac{ \sinh(\gamma_H - \alpha)^2}{R_x^2} - \frac{\beta \sigma_{\rm max}^2 \eta_t}{4} \right)$. 

Next, summing and telescoping gives us that
\begin{align*}
    \sum_{k=0}^{t-1} \norm{\vw_{k+1}-\vw}^2 - \norm{\vw_k - \vw}^2 &\leq \sum_{k=0}^{t-1} \Big[ 2 \eta_k \big(L_{\rm rob}(\vw; \Sc) - L_{\rm rob}(\vw_k; \Sc)\big)\; + \nonumber \\
    &\qquad \qquad \frac{\eta_k^2}{\bar{\eta}_k} \big(L_{\rm rob}(\vw_k; \Sc) - L_{\rm rob}(\vw_{k+1}; \Sc)\big) \Big] 
\end{align*}
or
\begin{align*}
    \norm{\vw_t - \vw}^2 - \norm{\vw_0 -\vw}^2 &\leq 2 \sum_{k=0}^{t-1} \eta_k \big(L_{\rm rob}(\vw; \Sc) - L_{\rm rob}(\vw_k; \Sc)\big) \; +  \nonumber \\
    &\qquad \sum_{k=0}^{t-1} \frac{\eta_k^2}{\bar{\eta}_k} \big(L_{\rm rob}(\vw_k; \Sc) - L_{\rm rob}(\vw_{k+1}; \Sc)\big). \\
\end{align*}
Now, multiplying both sides by $-1$ completes the proof as follow.
\begin{align*}
    &2\sum_{k=0}^{t-1} \eta_k \big(L_{\rm rob}(\vw_k; \Sc) - L_{\rm rob}(\vw; \Sc)\big)\; + \nonumber \\
    &\qquad \qquad \sum_{k=0}^{t-1} \frac{\eta_k^2}{\bar{\eta}_k} \big(L_{\rm rob}(\vw_{k+1}; \Sc) - L_{\rm rob}(\vw_k; \Sc)\big) \leq \norm{\vw_0 -\vw}^2 - \norm{\vw_t - \vw}^2\; .
\end{align*}
\end{proof}
\noindent We are now in a position to present the desired convergence result. 
\begin{theorem}[Convergence GD update, Algorithm~\ref{alg:adv-mf}]\label{conv:alg-2}
For a fixed constant $c \in (0, 1)$, let the step size 
$\eta_t := \eta = c\cdot \frac{2\sinh^2(\gamma_H - \alpha)}{\beta \sigma_{\rm max}^2 R_x^2}$
and $\mathcal{A}$ be the GD update as defined in \eqref{eq:adv-gd-1}. Then, the iterates $\{\vw_t\}$ in Algorithm~\ref{alg:adv-mf} satisfy 
\begin{align*}
         L_{\rm rob}(\vw_t; \Sc) = O \left(  \frac{\sinh^2(\ln(t))}{t}\cdot\left(\sinh(\gamma_H - \alpha)\right)^{-4}\right)
     \; .
 \end{align*}
\end{theorem}
\begin{proof}
Without loss of generality, assume that $\vw_0=(0,\ve_i)$ where again $\ve_i \in \R^{d}$ is a standard basis vector whose $i$-th coordinate is $1$ ($i>0$). Note that this is a valid initialization, since $\vw_0*\vw_0<0$; furthermore, we have $\norm{\vw_0}=1$. Let $\vw^\ast \in \R^{d+1}$ be a classifier that achieves the margin $\gamma_H$ on $\mathcal{S}$, i.e., $\forall (\vx, y) \in \Sc$,
\begin{align*}
    y(\vx*\vw^\ast) \geq \sinh(\gamma_H) \quad \iff  \quad \asinh \left(	\frac{y(\vw^{\ast}*\vx)}{\sqrt{- \vw^{\ast}*\vw^{\ast}}}\right) \geq \gamma_H.
\end{align*}
Without loss of generality, assume that $\norm{\vw^\ast}=1$. Let $\vu_t := \frac{\sinh(\ln (t))}{\sinh(\gamma_H - \alpha)} \vw^\ast$; then $\norm{\vu_t}=\frac{\sinh(\ln (t))}{\sinh(\gamma_H - \alpha)}$. We have
\begin{align}
\label{eq:loss_u}
    L_{\rm rob} (\vu_t; \Sc'_t) &= \frac{1}{\vert \mathcal{S}_t' \vert} \sum_{(\vx,y) \in \mathcal{S}_t'} l_{\rm rob} (\vx,y;\vu_t)
    = \frac{1}{\vert \mathcal{S}_t' \vert} \sum_{(\vx,y) \in \mathcal{S}_t'} f(y(\tilde{\vx}*\vu_t)) \nonumber \\
    &\overset{(i)}{\leq} \frac{1}{\vert \mathcal{S}_t' \vert} \sum_{(\tilde{\vx},y) \in \mathcal{S}_t'} f\big(\sinh(\ln (t))\big)
    = f\big(\sinh(\ln (t))\big) \nonumber \\
    &\overset{(ii)}{\leq} \ln \left(1 + \exp\left(-\ln(t)\right)\right) \overset{(iii)}{\leq} \frac{1}{t} \; ,
\end{align}
where $(i)$ follows from 
\begin{align*}
    y(\tilde{\vx}*\vu_t) = \frac{\sinh(\ln (t))}{\sinh(\gamma_H - \alpha)} \underbrace{y(\tilde{\vx}*\vw^\ast)}_{\geq \sinh(\gamma_H-\alpha)} \geq \sinh(\ln(t)) \; ,
\end{align*}
(ii) from $\sinh(x) \geq x$ for all $x \geq 0$ and $(iii)$ follows from the fact that $\ln(1+x) \leq x$.

\noindent Now, consider
\begin{align*}
   &2 \eta (t-1) \big(L_{\rm rob}(\vw_t; \Sc) - L_{\rm rob}(\vu_t; \Sc)\big) \overset{(iv)}{=} 2 \sum_{k=0}^{t-1} \eta_k \big(L_{\rm rob}(\vw_{t}; \Sc) - L_{\rm rob}(\vu_t; \Sc)\big)\\
    &\qquad=2\sum_{k=0}^{t-1} \eta_k  \big(L_{\rm rob}(\vw_{t}; \Sc) - L_{\rm rob}(\vu_t; \Sc) + L_{\rm rob}(\vw_{k}; \Sc) - L_{\rm rob}(\vw_{k}; \Sc)\big) \\
    &\qquad= 2 \sum_{k=0}^{t-1} \eta_k \big(L_{\rm rob}(\vw_k; \Sc) - L_{\rm rob}(\vu_t; \Sc)) + 2 \sum_{k=0}^{t-1} \eta_k \big(L_{\rm rob}(\vw_{t}; \Sc) - L_{\rm rob}(\vw_{k}; \Sc)\big) \\
  & \qquad\overset{(v)}{\leq} 2 \sum_{k=0}^{t-1} \eta_k \big(L_{\rm rob}(\vw_k; \Sc) - L_{\rm rob}(\vu_t; \Sc)) +  \sum_{k=0}^{t-1} \eta_k \big(L_{\rm rob}(\vw_{k+1}; \Sc) - L_{\rm rob}(\vw_{k}; \Sc)\big) \\
  &\qquad\text{$\leq 2 \sum_{k=0}^{t-1} \eta_k \big(L_{\rm rob}(\vw_k; \Sc) - L_{\rm rob}(\vu_t; \Sc)) +  \sum_{k=0}^{t-1} \frac{\eta_k^2}{\bar{\eta}_k} \big(L_{\rm rob}(\vw_{k+1}; \Sc) - L_{\rm rob}(\vw_{k}; \Sc)\big)$} \\
    &\qquad\overset{(vi)}{\leq} \norm{\vw_0 - \vu_t}^2 - \norm{\vw_t - \vu_t}^2 \; ,
\end{align*}
where $(iv)$ holds as we have a constant step-size, i.e., $\eta_k = \eta$ and $(v)$ follows from the fact that $$L_{\rm rob}(\vw_{t};\Sc) \leq L_{\rm rob}(\vw_{k+1};\Sc)~\quad \text{for}~0 \leq k \leq t-1.$$ 
We can rewrite this as
\begin{align*}
    L_{\rm rob} (\vw_t; \Sc) &\leq L_{\rm rob}(\vu_t; \Sc) + \frac{\norm{\vw_0 - \vu_t}^2 - \norm{\vw_t - \vu_t}^2}{2 \sum_{k=0}^{t-1} \eta_k} \\
    &\overset{(i)}{\leq} \frac{1}{t} + \frac{\norm{\vw_0 - \vu_t}^2}{2(t-1)\eta} \\
    & \overset{(ii)}{\leq} \frac{1}{t} + \frac{2\norm{\vw_0}^2 + 2\norm{\vu_t}^2}{2(t-1)\eta} = \frac{1}{t} + \frac{\norm{\vw_0}^2 + \norm{\vu_t}^2}{(t-1)\eta}
\end{align*}
where $(i)$ follows from \eqref{eq:loss_u} and $(ii)$ follows from $(a+b)^2 \leq 2a^2 + 2b^2$. Now, using the fact that $\norm{\vw_0} = 1$ and $\norm{\vu_t} = \frac{sinh(\ln (t))}{\sinh(\gamma_H - \alpha)}$, we obtain that
\begin{align}
\label{eq:thm-gradient-pre-final}
        L_{\rm rob}(\vw_t; \Sc) \leq \frac{1}{t} +\frac{1 +  2\left(1/{\sinh(\gamma_H - \alpha)}\right)^{2}\cdot \sinh^2(\ln(t))}{(t-1)\eta} 
    \; .
\end{align}
By substituting $\eta = c\cdot \frac{2\sinh^2(\gamma_H - \alpha)}{\beta \sigma_{\rm max}^2 R_x^2}$, we get
\begin{align*}
         L_{\rm rob}(\vw_t; \Sc) = O \left(  \frac{\sinh^2(\ln(t))}{t} \cdot\left(\sinh(\gamma_H - \alpha)\right)^{-4} \right)
     \; .
 \end{align*}
\end{proof}
%
\begin{theorem}[Iteration complexity]
\label{thm:main-iteration-complexity}
Consider Algorithm~\ref{alg:adv-mf}  with $\eta_t:=\eta = c\cdot \frac{2\sinh^2(\gamma_H - \alpha)}{\beta \sigma_{\rm max}^2 R_x^2}$ and $\mathcal{A}$ being the GD update. Then Algorithm~\ref{alg:adv-mf} converges as $\Omega \left( {\rm poly}\left(t, \sinh(\gamma_H - \alpha) \right)   \right)$.
\end{theorem}
\begin{proof}
Let $\varrho = \frac{\ln(1 + 1/e)}{\ln(1 + e)}$. {We first argue that 
\begin{align}
\label{eq:l_rob_ub}
    L_{\rm rob}(\vw_t; \Sc) \leq \varrho \cdot \ln \left( 1 + \exp \left( -(\gamma_H - \alpha)   \right) \right)
\end{align}
 implies that $\vw_t$ achieves margin $\gamma_H-\alpha$ on $\mathcal{S}$.} To see this, note that
\begin{align}
\label{eq:l_rob_lb}
    L_{\rm rob}(\vw_t; \Sc) &= \frac{1}{\vert \Sc \vert} \sum_{(\vx,y) \in \Sc} l_{\rm rob} (\vx,y; \vw_t) \; \nonumber \\
    & = \underbrace{\max_{(\vx,y) \in \Sc}l_{\rm rob} (\vx,y; \vw_t)}_{:=l^{\rm max}_{\rm rob}(\Sc)}\cdot \frac{1}{\vert \Sc \vert} \sum_{(\vx,y) \in \Sc} \frac{l_{\rm rob} (\vx,y; \vw_t)}{l^{\rm max}_{\rm rob}(\Sc)} \; \nonumber \\
    & \geq  l^{\rm max}_{\rm rob}\cdot\frac{1}{\vert \Sc \vert}\sum_{(\vx,y) \in \Sc} {\varrho} =  {\varrho}\cdot l^{\rm max}_{\rm rob}.
\end{align}

The last inequality in \eqref{eq:l_rob_lb} holds as, for each $(\vx, y) \in \Sc$, we have
\begin{align*}
\ln (1 + {1}/{e}) \overset{(i)}{\leq} 
\underbrace{\ln \left(1 + \exp\left(-\left(y(\tilde{\vx}\ast\vw_t) \right)\right)\right)}_{=l_{\rm rob}(\vx, y; \vw_t)} \leq 
  \max_{(\vx,y) \in \Sc}l_{\rm rob} (\vx,y; \vw_t) \overset{(ii)}{\leq} \ln(1 + e),
\end{align*}
where $(i)$ and $(ii)$ follows from  Assumption~\ref{ass1}.2. Thus, for each $(\vx, y) \in \Sc$, we have 
\begin{align}
\frac{l_{\rm rob}(\vx, y; \vw_t)}{l^{\rm max}_{\rm rob}} \geq \frac{\ln(1 + 1/e)}{\ln(1 + e)} = {\varrho}.
\end{align}
Now, by combining \eqref{eq:l_rob_ub} and \eqref{eq:l_rob_lb}, we obtain that 
\begin{align*}
l_{\rm rob}(\vx,y;\vw_t)
&\leq \ln \left(1+\exp \left(- (\gamma_H - \alpha)\right)\right) 
\end{align*} 
for any $(\vx,y) \in \mathcal{S}$. Equivalently, for each $(\vx, y) \in \Sc$,
\begin{align}\label{eq:bound-lrob}
    l_{\rm rob}(\vx,y;\vw_t)&=\ln \left(1 + \exp\left(-\left(y(\tilde{\vx}\ast\vw_t) \right)\right)\right)\nonumber \\
    &\leq \ln \left(1+\exp \left(- (\gamma_H - \alpha)\right)\right).
\end{align}
Thus, for each $(\vx, y) \in \Sc$, we have
\begin{align*}
    \asinh\left(\frac{y(\tilde{\vx}\ast\vw_t)}{\sqrt{-\vw_t*\vw_t}} \right) 
    \overset{*}{\geq}  \gamma_H-\alpha \; .
\end{align*}
where (*) follows from Eq.~\ref{eq:bound-lrob}. {Thus, $\vw_t$ achieves margin $\gamma_H- \alpha$ on $\mathcal{S}$.}

Next, introduce the following constant:
\begin{align*}
    C_q := \inf \lbrace t \geq 2: \; 2 + \ln(t)^2 \leq (t-1) t^{-1/q} \rbrace \; .
\end{align*}
With this, for $t \geq C_q$, we can rewrite the bound in \eqref{eq:thm-gradient-pre-final} as follows:
\begin{align*}
    L_{\rm rob}(\vw_t; \Sc) &\leq \underbrace{\frac{1}{t}}_{\leq \frac{1}{(t-1)\eta}} + \frac{1+\sinh(\ln(t))^2 \sinh \left(\gamma_H-\alpha\right)^{-2}}{(t-1)\eta } \leq \frac{2 + \sinh(\ln(t))^2 \sinh \left(\gamma_H- \alpha \right)^{-2}}{(t-1)\eta}\\
    &\leq \frac{2 + \ln(t)^2 \sinh \left(\gamma_H - \alpha \right)^{-2}}{(t-1)\eta} \leq \frac{(t-1)t^{-1/q}}{\eta (t-1)} \sinh \left(\gamma_H- \alpha \right)^{-2} \leq \frac{t^{-1/q}}{\eta \sinh \left( \gamma_H-\alpha \right)^2} \; .
\end{align*}
Solving for $t$ and plugging in the above bound on $L_{\rm rob}$ for which $\vw_t$ achieves the desire margin, as well as $\eta = c\cdot \frac{2\sinh^2(\gamma_H - \alpha)}{\beta \sigma_{\rm max}^2 R_x^2 }$, we get 
\begin{align*}
    t= \max \lbrace C_q, \Omega \left( \left( \left( \sinh(\gamma_H - \alpha)^4/  \right) \right)^{-q}  \right) \rbrace \; ,
\end{align*}
from which the claim follows directly.
\end{proof}

\subsection{Algorithm~\ref{alg:adv-mf} with an ERM update}
\label{sec:C3}

Consider the unit sphere $\mathbb{S}^{d-1} \subseteq \R^d$. A spherical code with minimum separation $\theta$ is a subset of $\mathbb{S}^{d-1}$, such  that any two distinct elements $\vu, \vu'$ in the subset are separated by at least an angle $\theta$, i.e. $\ip{\vu}{\vu'} \leq \cos \theta$. We denote the size of the largest such code as $A(d, \theta)$. A similar construction can be made in hyperbolic space, which allows the transfer of bounds on $A(d,\theta)$ to hyperbolic space~\citep{cohn}.

The following lemma shows that a spherical code with a suitable minimum separation $\theta$ enables a simple pathological training set such that Algorithm~\ref{alg:adv-mf} along with an ERM update rule cannot produce a classifier with a desired margin in a small number of iteration. In particular, the lemma shows that the number of iterations required to find the desire margin is lower-bounded by the size of the underlying spherical code.

\begin{lem}\label{lem:sph-codes}
Consider $\mathcal{S}=\lbrace	(\vx_1, y_1) = \big((1,0,\dots,0),  1\big), (\vx_2, y_2) = \big((-1,0,\dots,0), -1\big)\rbrace$, where $\vx_1, \vx_2 \in \L^d$ and $y_1,y_2$ the corresponding labels. For any $\epsilon<\alpha$, there is an admissible sequence of classifiers $\{\vw_t\}_{1 \leq t \leq T}$, with
\begin{align*}
T= A\left(d, \arccos\Big(\rho\cdot \frac{\sinh(\epsilon) \cosh(\alpha)}{\sqrt{\cosh^2 (\alpha) -1 } \sqrt{1 + \sinh^2(\epsilon)}}\Big)\right)
\end{align*}
\end{lem}
\begin{proof}
First, note that $\vx_1 * \vx_1 = \vx_2 * \vx_2 = 1$, i.e., $\vx_1, \vx_2 \in \L^d$ as desired. Let $\epsilon' = \sinh(\epsilon)$ and $\ve_i \in \R^{d+1}$ denotes the standard basis vector that has its $i$-th coordinate equal to $1$. Now, consider classifiers of the form
\begin{align}
\label{eq:w_seq}
\vw_t = \begin{pmatrix}
\epsilon' \\
\sqrt{1 + \epsilon'^2} ~\vv_t
\end{pmatrix}
\quad
{\rm where}
\quad
\vv_t \in \mathcal{C}\left(d, \arccos\Big(\rho\cdot \frac{\epsilon' \sqrt{1 + \delta^2}}{\delta \sqrt{1 + \epsilon'^2}}\Big)\right) \quad \forall \; 1 \leq t \leq T \;,
\end{align}
where $\rho < 1$; and $\mathcal{C}\left(d, \arccos\Big(\rho\cdot\frac{\epsilon' \sqrt{1 + \delta^2}}{\delta \sqrt{1 + \epsilon'^2}}\Big)\right)$ be the spherical code with the minimum separation $\theta =\arccos\Big(\rho\cdot\frac{\epsilon' \sqrt{1 + \delta^2}}{\delta \sqrt{1 + \epsilon'^2}}\Big)$ and size $A\big(d, \theta\big)$. Since $\vw_t * \vw_t = (\epsilon')^2 - 1 - (\epsilon')^2 = -1$, we have $\vw_t * \vw_t < 0$ for all $t$. This guarantees that the intersections of the decision boundaries defined by $\{\vw_t\}_t$ and $\L^d$ are not empty. Moreover, $\{\vw_t\}$ is an admissible sequence of classifiers with margin $\leq \epsilon$. To see this, note that, for $t=1, \dots, T$,
\begin{align*}
\vw_t * \vx_1 &= \epsilon' > 0 \\
\vw_t * \vx_2 &= - \epsilon' < 0 \; ,
\end{align*}
i.e., $\{\vw_t\}$ correctly classifies $\Sc$. Furthermore, with $-\vw_t * \vw_t = 1$, we have
\begin{align*}
\asinh \left(	\frac{y_1 (\vw_t*\vx_1)}{\sqrt{-\vw_t*\vw_t}} \right)=\asinh \left(	\frac{y_2 (\vw_t*\vx_2)}{\sqrt{-\vw_t*\vw_t}} \right)=\epsilon \; ,
\end{align*}
which gives ${\rm margin}_{\Sc} (\vw_t)=\epsilon$.

Now we perturb $\vx_1, \vx_2$ on $\L^d$ such that the magnitude of the perturbation is at most $\alpha$, i.e., we want to find $\tilde{\vx}_1, \tilde{\vx}_2 \in \L^d$ such that both $d_{\L} (\vx_1, \tilde{\vx}_1)$ and $d_{\L} (\vx_2, \tilde{\vx}_2)$ are at most $\alpha$. For $1 \leq t \leq T$, consider adversarial examples of the form
\begin{align*}
\tilde{\vx}_{1t}= \begin{pmatrix}
\sqrt{1+\delta^2} \\
\delta \vv_t
\end{pmatrix} 
\quad \text{and} \quad
\tilde{\vx}_{2t}= -\begin{pmatrix}
\sqrt{1+\delta^2} \\
\delta \vv_t
\end{pmatrix}.
\end{align*}
Note that $\tilde{\vx}_{1t},\tilde{\vx}_{2t} \in \L^d$ as $\tilde{\vx}_{1i} * \tilde{\vx}_{1t} = \tilde{\vx}_{2t} * \tilde{\vx}_{2i} = 1$. Let us verify the two conditions that we require the valid adversarial examples to satisfy:
\begin{itemize}
    \item \textbf{Adversarial budget.}~Note that we have
\begin{align*}
d_\L (\vx_1, \tilde{\vx}_{1t}) = d_\L (\vx_2, \tilde{\vx}_{2t}) = \acosh(\sqrt{1+\delta^2}) \;.
\end{align*}
Thus, by choosing $\delta = \sqrt{\cosh^2 (\alpha) -1 }$, we achieve the maximal permitted perturbation $\alpha$. 
\item \textbf{Inconsistent prediction for the current classifier, i.e., $h_{\vw_t}(\tilde{\vx}_{1t/2t}) \neq h_{\vw}(\vx_{1/2})$.}~Note that we have $\delta \geq \alpha >\epsilon$, which further implies that $\delta > \epsilon \geq \epsilon'$. In round $t$,
\begin{align*}
\vw_t * \tilde{\vx}_{1t} = \epsilon' \sqrt{1 + \delta^2} - \delta \sqrt{1+\epsilon'^2} < 0 \\
\vw_t * \tilde{\vx}_{2t} = - \epsilon' \sqrt{1 + \delta^2} + \delta \sqrt{1+\epsilon'^2} > 0 \; ,
\end{align*}
which is a consequence of the relation  $\delta > \epsilon'$ as follows:
\begin{align*}
\delta^2 > \epsilon'^2 \Rightarrow \delta^2 + \epsilon'^2 \delta^2 > \epsilon'^2 + \epsilon'^2 \delta^2 &\Rightarrow \delta^2(1+\epsilon'^2) > \epsilon'^2 (1+\delta^2) \nonumber \\ &\Rightarrow \delta \sqrt{1+\epsilon'^2} > \epsilon' \sqrt{1+\delta^2} \; .
\end{align*}
\end{itemize}
Recall that, in each round of Algorithm~\ref{alg:adv-mf} with an ERM update, we create adversarial examples and add them to the training set, i.e., after round $t$ we have
\begin{align*}
\Sc_{<t} = \Sc \cup \bigcup_{i=0}^{t-1} \lbrace (\tilde{\vx}_{1i}, y_{1i}), (\tilde{\vx}_{2i}, y_{2i})	\rbrace \; .
\end{align*}

\noindent Now for each $t$ and any $i<t$, we have
\begin{align*}
\vw_t * \tilde{\vx}_{1i} &= \epsilon' \sqrt{1 + \delta^2} - \delta\sqrt{1 + \epsilon'^2}\cdot \cos(\theta) > 0 \\
\vw_t * \tilde{\vx}_{2i} &= -\epsilon' \sqrt{1 + \delta^2} + \delta\sqrt{1 + \epsilon'^2}\cdot \cos(\theta)< 0 \; ,
\end{align*}
i.e., $\vw_t$ linearly separates $\Sc_{<t}$. 

Therefore, $\{\vw_t\}$ in \eqref{eq:w_seq} form an admissible sequence of the classifiers, where $\vw_t$ linearly separates $\Sc_t$ while achieving the margin of at most $\epsilon$ on the original dataset $\Sc$. The length of the sequence is bounded by the size of the spherical code $\mathcal{C}\big(d, \epsilon'\cosh(\alpha)\big)$, which give us that
\begin{align*}
T= A\left(d, \arccos\Big(\rho\cdot \frac{\epsilon' \sqrt{1 + \delta^2}}{\delta \sqrt{1 + \epsilon'^2}}\Big)\right) \;= A\left(d, \arccos\Big(\rho\cdot \frac{\sinh(\epsilon) \cosh(\alpha)}{\sqrt{\cosh^2 (\alpha) -1 } \sqrt{1 + \sinh^2(\epsilon)}}\Big)\right) \; .
\end{align*}
\end{proof}
The following result (a restatement of Theorem~\ref{prop:erm} from the main text) then follows by applying a lower bound on the maximal size of spherical codes by Shannon.
\begin{theorem}[Theorem~\ref{prop:erm}]
\label{prop:erm-appen}
Suppose Algorithm~\ref{alg:adv-mf} (with an ERM update) outputs a linear seperator of $\Sc \cup \Sc'$. In the worst case, the number of iteration required to achieve the margin at least $\epsilon$ is $\Omega\left( \exp(d) \right)$.
\end{theorem}
\begin{proof}
The statement of the theorem follows from combining Lemma~\ref{lem:sph-codes} with Shannon's lower bound (Theorem~\ref{thm:sh-bound}) on the maximal size of spherical codes, namely
\begin{align*}
T \geq (1+ o(1)) \sqrt{2 \pi d} \frac{\cos(\theta)}{\sin^{d-1} (\theta)} \; .
\end{align*}
We introduce the shorthand $\theta =: \arccos \left(\frac{A}{B}  \right)$, where $A = \rho \sinh(\epsilon) \cosh(\alpha)$ and \\ 
 $B= \sqrt{\cosh^2(\alpha) - 1}\sqrt{1+\sinh^2(\epsilon)}$, as given by Lemma~\ref{lem:sph-codes}. We then use two well-known trigonometric identities
\begin{align*}
    \cos (\arccos z) = z\quad\text{and}\quad
    \sin (\arccos z) = \sqrt{1-z^2} 
\end{align*}
to simplify the trignometric fraction in Shannon's bound:
\begin{align*}
    \frac{\cos \theta}{\sin^{d-1} \theta} = \frac{A}{B \left(
    1 - \frac{A^2}{B^2} \right)^{\frac{d-1}{2}} } 
    = \frac{AB^{d-2}}{(B^2 - A^2)^{\frac{d-1}{2}}} \; .
\end{align*}
For the denominator, note that
\begin{align*}
    B^2 - A^2 &= (\cosh^2(\alpha) - 1)(1+\sinh^2(\epsilon)) - \rho^2 \sinh^2(\epsilon) \cosh^2(\alpha) \\
    &= (1- \rho^2) \sinh^2(\epsilon) \cosh^2(\alpha) + \cosh^2(\alpha) -1 - \sinh^2(\epsilon) \\
    &\overset{(i)}{\simeq} \cosh^2(\alpha) - 1 - \sinh^2(\epsilon)
\end{align*}
where (i) follows from the fact that we can choose $\rho$ arbitrary close to 1. Putting everything together, we have the lower bound
\begin{align*}
    T \geq (1+o(1)) \sqrt{2d} \frac{\rho \sinh(\epsilon) \cosh(\alpha) \left( \sqrt{\cosh^2(\alpha) - 1} \sqrt{1+\sinh^2(\epsilon)}     \right)^{d-2}}{\left( \cosh^2(\alpha) - 1 - \sinh^2(\epsilon)      \right)^{\frac{d-1}{2}}} = \Omega (\exp d) \; ,
\end{align*}
which is exponential in $d$.
\end{proof}

\section{Dimension-distortion trade-off}
\subsection{Euclidean case}
\label{sec:D1}
In the Euclidean case, we relate the distance of the support vectors and the size of margin via side length - altitude relations. Let $\vx,\vy \in \R^{d}$ denote support vectors, such that $\ip{\vx}{\vw} > 0$ and $\ip{\vy}{\vw} <0$ and ${\rm margin}(\vw)=\epsilon$. We can rotate the decision boundary, such that the support vectors are not unique. Wlog, assume that $\vx_1, \vx_2$ are equidistant from the decision boundary and $\norm{\vw}=1$. In this setting, we show the following relation:
\begin{theorem}[Thm.~\ref{prop:e-dist}]
$\epsilon' \geq \frac{\epsilon}{c_E^3}$.
\end{theorem}
\begin{proof}
Let $d_1 = d_\Xc (\phi_E^{-1}(\vx_1),\phi_E^{-1}(\vy))$, $d_2 = d_\Xc (\phi_E^{-1}(\vx_2),\phi_E^{-1}(\vy))$ and $d_3 = d_\Xc (\phi_E^{-1}(\vx_1), \phi_E^{-1}(\vx_2))$ the distances between the support vectors in the original space. In the Euclidean embedding space we have
\begin{align*}
d_1' &= d_E(\vx_1, \vy) \geq \frac{d_1}{c_E} \\
d_2' &= d_E(\vx_2, \vy) \geq \frac{d_2}{c_E} \\
d_3' &= d_E(\vx_1, \vx_2) \geq \frac{d_3}{c_E} \; .
\end{align*}
$d_1', d_2', d_3'$ are the side lengths of a triangle, whose altitude is given by the margin: $h=2 \epsilon'$. With Heron's equation we get
\begin{align*}
h = 2 \epsilon' = \frac{2}{d_3'} \sqrt{s'(s'-d_1')(s'-d_2')(s'-d_3')} \; ,
\end{align*}
where $s'=\frac{1}{2}(d_1' + d_2' + d_3')$. In $\Xc$ we have $s' = \frac{1}{2 c_E}(d_1 + d_2 + d_3) = \frac{s}{c_E}$. Then we have with respect to the actual distance relations
\begin{align*}
h = 2 \epsilon' \geq \frac{2 }{c_E d_3} \sqrt{c_E^{-4} s(s-d_1)(s-d_2)(s-d_3)} = 2  \frac{ \epsilon}{c_E^3} \; ,
\end{align*}
which gives the claim.
\end{proof}

\subsection{Hyperbolic case}
\label{sec:D2}
As in the Euclidean case, we want to relate the margin to the distance of the support vectors. Since the distortion can be expressed in terms of the distances of support vector in the original and the embedding space, this allows us to study the influence of distortion on the margin.

We will derive the relation in the half-space model ($\Pb^2$). However, since the theoretical guarantees above consider the upper sheet of the Lorentz model ($\L_{+}^{d'}$), we have to map between the two spaces.\\

\begin{ass}\label{ass:dist}
We make the following assumptions on the underlying data $\Xc$ and the embedding $\phi_H$:
\begin{enumerate}
    \item $\Xc$ is linearly separable;
    \item $\Xc$ is hierarchical, i.e., has a partial order relation;
    \item $\phi_H$ preserves the partial order relation and the root is mapped onto the origin of the embedding space.
\end{enumerate}
\end{ass}

\noindent Under these assumptions, the hyperbolic embedding $\phi_H$ has two sources of distortion:
\begin{enumerate}
    \item the (multiplicative) distortion of pairwise distances, measured by the factor $\frac{1}{c_H}$;
    \item the distortion of order relations, in most embedding models captured by the alignment of ranks with the Euclidean norm. 
\end{enumerate}
Under Ass.~\ref{ass:dist}, order relationships are preserved and the root is mapped to the origin. Therefore, the distortion on the Euclidean norms is given as follows:
\begin{align*}
    \norm{\phi_H(x)} = d_E(\phi_H(\vx),\phi_H(0)) = \frac{d_\Xc(\vx,0)}{c_H} \; ,
\end{align*}
i.e., the distortion on both pairwise distances and norms is given by a factor $\frac{1}{c_H}$.\\

\noindent \emph{Note on notation: }In the following, a bar over any symbol indicates the Euclidean expression.

\subsubsection{Mapping from $\L_{+}^{d'}$ to $\Pb^2$}
First, note that a transformation $\vv \mapsto B\vv$ with $B=\begin{pmatrix}
1 &0 \\
0 &A
\end{pmatrix}$ and an orthogonal matrix $A$ is isometric, i.e., it preserves the Minkowski product~\citep{cho}:
\begin{align*}
(B\vu)*(B\vv) = u_0 v_0 - \vu_{1:d'}^T A^T A \vv_{1:d'} = u_0 v_0 - \vu_{1:d'}^T \vv_{1:d'} = \vu*\vv \; .    
\end{align*}
Setting the first column of $A$ to $\frac{\vw_{1:d'}}{\norm{\vw_{1:d'}}}$ we can isometrically transform the decision hyperplane as 
             $\hat{\vw}=B \vw = (\hat{\vw}_0, \norm{\hat{\vw}_{1:d'}},0,\dots,0)$. Analogously, we can transform any point in $\L_{+}^{d'}$. In the following, we will use the shorthand $\lambda = \frac{\hat{w}_0}{\hat{w}_1}$. We can then use the maps defined in section~\ref{sec:equiv} to map $\hat{\vx}=Bx \in \L_{+}^2$ onto $\vz \in \Pb^2$, i.e. applying $(\pi_{BP} \circ (\pi_{LB} \circ B))$ to any $\vx \in \L_{+}^2$ gives $\vz \in \Pb^2$.

\begin{rmk}[Effect of hyperbolic distortion on Euclidean distances in the Poincare half plane]\normalfont
Note that the hyperbolic distance in the Poincare half plane can be written as follows:
\begin{align*}
    d_{\Pb}((x_0,x_1),(y_0,y_1)) &= 2 \asinh \left( \frac{1}{2} \sqrt{\frac{(x_0 - y_0)^2 + (x_1 - y_1)^2}{x_1 y_1}} \right) \\
    &= 2 \asinh \left( \frac{1}{2} \frac{d_E((x_0,x_1),(y_0,y_1))}{\sqrt{x_1 y_1}}  \right) \; .
\end{align*} 
If $c_H$ denotes the hyperbolic distortion, we get
\begin{align*}
    &d_{\Pb}' = \frac{d_{\Pb}}{c_H} = 2 \asinh \left( \frac{1}{2} \frac{d_E'}{\sqrt{x_1 y_1}}          \right) \\
    \Rightarrow \quad &\frac{1}{2} \frac{d_E'}{\sqrt{x_1 y_1}} = \sinh \left( \frac{2 \asinh \left(  \frac{1}{2} \frac{d_E}{\sqrt{x_1 y_1}} \right)}{2 c_H}\right) \gtrsim \frac{1}{2} \frac{d_E}{c_H \sqrt{x_1 y_1}} \; .
\end{align*}        
This suggests, that the effect of hyperbolic distortion on the Euclidean distances can be quantified by a comparable factor, i.e. $d_E' \gtrsim \frac{d_E}{c_H}$.
\end{rmk}
\begin{lem}[Relation between h-margin and E-margin]
\label{prop:e-h-margin}
Let $\gamma_H$ be the margin of a hyperbolic classifier $\vw \in \mathbb{R}^{d'+1}$. Then the Euclidean margin $\gamma_E$ of $\vw$ is bounded as follows: $\gamma_E \geq \sinh(\gamma_H)$.
\end{lem}
\begin{proof}
We again write the hyperbolic distance in the Poincare half plane in terms of the Euclidean distance of the ambient space:
\begin{align*}
    d_{\Pb}((x_0,x_1),(y_0,y_1)) &= 2 \asinh \left( \frac{1}{2} \sqrt{\frac{(x_0 - y_0)^2 + (x_1 - y_1)^2}{x_1 y_1}} \right) \\
    &= 2 \asinh \left( \frac{1}{2} \frac{d_E((x_0,x_1),(y_0,y_1))}{\sqrt{x_1 y_1}}  \right) \; ,
\end{align*} 
where $\vy \in \mathcal{H}_w$ is the point closest to the support vector $\vx \in \L_{+}^{d'}$ on the decision boundary. Therefore, the hyperbolic margin is $d_\Pb(\vx,\vy)=\gamma_H$ and the Euclidean margin is $d_E(\vx,\vy)=\gamma_E$.

Since we mapped the feature space onto the Poincare half plane, $\vy$ has the coordinates $\vy=(\tilde{y}_0, \tilde{y}_1, 0, \dots, 0)$ where $\tilde{y}_0=y_0$ and $\tilde{y}_1=\frac{\vw'^T \vy'}{\norm{\vw'}}$. Similarly, $\vx$ has the coordinates $\vx=(\tilde{x}_0, \tilde{x}_1, 0, \dots, 0)$. The transformation preserves the Minkowski product. Therefore we have
\begin{align*}
    \vy*\vy &= y_0^2 - \vy'^2 = \hat{y}_0^2 - \Big( \underbrace{\frac{\vw'^T \vy'}{\norm{\vw'}}}_{=\hat{y}_1} \Big)^2 = 1
\end{align*}
and similarly $\vx*\vx=\hat{x}_0^2 - \hat{x}_1^2 = 1$. This implies
\begin{align*}
    \circled{1} \quad 
        \hat{y}_1 &= \sqrt{\hat{y}_0^2-1}, \; \hat{x}_1 = \sqrt{\hat{x}_0^2-1} \; ,
\end{align*}
and further
\begin{align*}
    \circled{2} \quad \hat{x}_0, \; \hat{y}_0 \geq 1 \; .
\end{align*}
We want to show that $\hat{x}_1 \hat{y}_1 \geq 1$. For this, first, note that since $\vy \in \mathcal{H}_w$ and the hyperbolic margin is $\gamma_H$, we have
\begin{align*}
    0 &= \vw*\vy = w_0 y_0 - \vw'^T \vy' \\
    \Rightarrow \vw'^T \vy' &= w_0 y_0 \; .
\end{align*}
This gives
\begin{align*}
    \vy*\vy &= y_0^2 - \frac{w_0^2 y_0^2}{\norm{\vw'}^2} = 1 \\
    \Rightarrow 0 &= y_0^2 - \frac{w_0^2 y_0^2}{\norm{\vw'}^2} - 1 \; ,
\end{align*}
and therefore
\begin{align*}
    \circled{3} \quad y_0 = \frac{1}{\sqrt{1 - \frac{w_0^2}{\norm{\vw'}}}} \; . 
\end{align*}
Since the hyperbolic margin is $\gamma_H$, we further have 
\begin{align*}
   d_{\Pb}(\vx,\vy)&=\acosh(\vx*\vy) \geq \gamma_H \quad
   \Rightarrow \; \vx*\vy \geq \cosh(\gamma_H) \geq 1 \; ,
\end{align*}
and therefore
\begin{align*}
    x_0 y_0 - x_1 y_1 &\geq 1 \\
    x_0 y_0 - \sqrt{x_0^2 - 1} \sqrt{y_0^2 -1 } &\geq 1 \\
    (x_0 y_0 -1)^2 &\geq (x_0^2-1)(y_0^2-1) \\
    x_0^2 y_0^2 - 2 x_0 y_0 +1 &\geq x_0^2 y_0^2 - x_0^2 - y_0^2 +1 \\
    \Rightarrow 0 &\leq (x_0 - y_0)^2 \; ,
\end{align*}
which implies
\begin{align*}
    \circled{4} \quad x_0 \geq y_0 \;.
\end{align*}
This gives for $x_1 y_1$ the following:
\begin{align*}
    x_1 y_1 &\overset{\circled{1}}{=} \sqrt{x_0^2 - 1} \sqrt{y_0^2 -1} \overset{\circled{4}}{\geq} y_0^2 -1 \overset{\circled{3}}{=} \frac{1}{1 - \frac{w_0^2}{\norm{\vw'}}} - 1 = \frac{w_0^2}{\norm{\vw'}^2 - w_0^2} \; .
\end{align*}
By assumption we have $\vw*\vw=w_0^2 - \norm{\vw'}^2=-1$, which gives for the denominator $-w_0^2 + \norm{\vw'}^2=1$. It remains to show that $w_0^2 \geq 1$. 

For this last step, we want to show that mass concentrates on $w_0$ as the classifier is updated, ensuring $w_0 \geq 1$. By construction, we have initially $\vw*\vw=-1$. Wlog, assume that initially $w_0 \geq 1$. An initialization of this form can always be found, e.g., by setting $\vw=(a, \sqrt{1+a^2},0, \dots, 0)$ for some $a \geq 0$. If the $i^{th}$ update is negative ($y^i x_0^i < 0$), then $\vert \vw \vert_*$ will initially decrease, but the normalization step will scale away the effect on $w_0$. However, if the $i^{th}$ update is non-negative ($y^i x_0^i \geq 0$), it will increase $w_0$. Over time, the positive updates concentrate the mass on $w_0$. Since we initialized to $w_0 \geq 1$, the condition will always stay valid. With the arguments above, this implies $x_1 y_1 \geq 1$. Inserting the latter in the expression above, we get
\begin{align*}
    d_H &= 2 \asinh \left( \frac{1}{2} \frac{d_E}{\sqrt{x_1 y_1}}  \right)
    \leq 2 \asinh \left( \frac{d_E}{2}  \right) \\
    \Rightarrow d_E &\geq 2 \sinh \left(\frac{d_H}{2}     \right) \geq \sinh(d_H) \; .
\end{align*}    
\end{proof}

\subsubsection{Characterizing the margin}
\begin{figure}[ht]
    \centering
    \includegraphics[width=7cm]{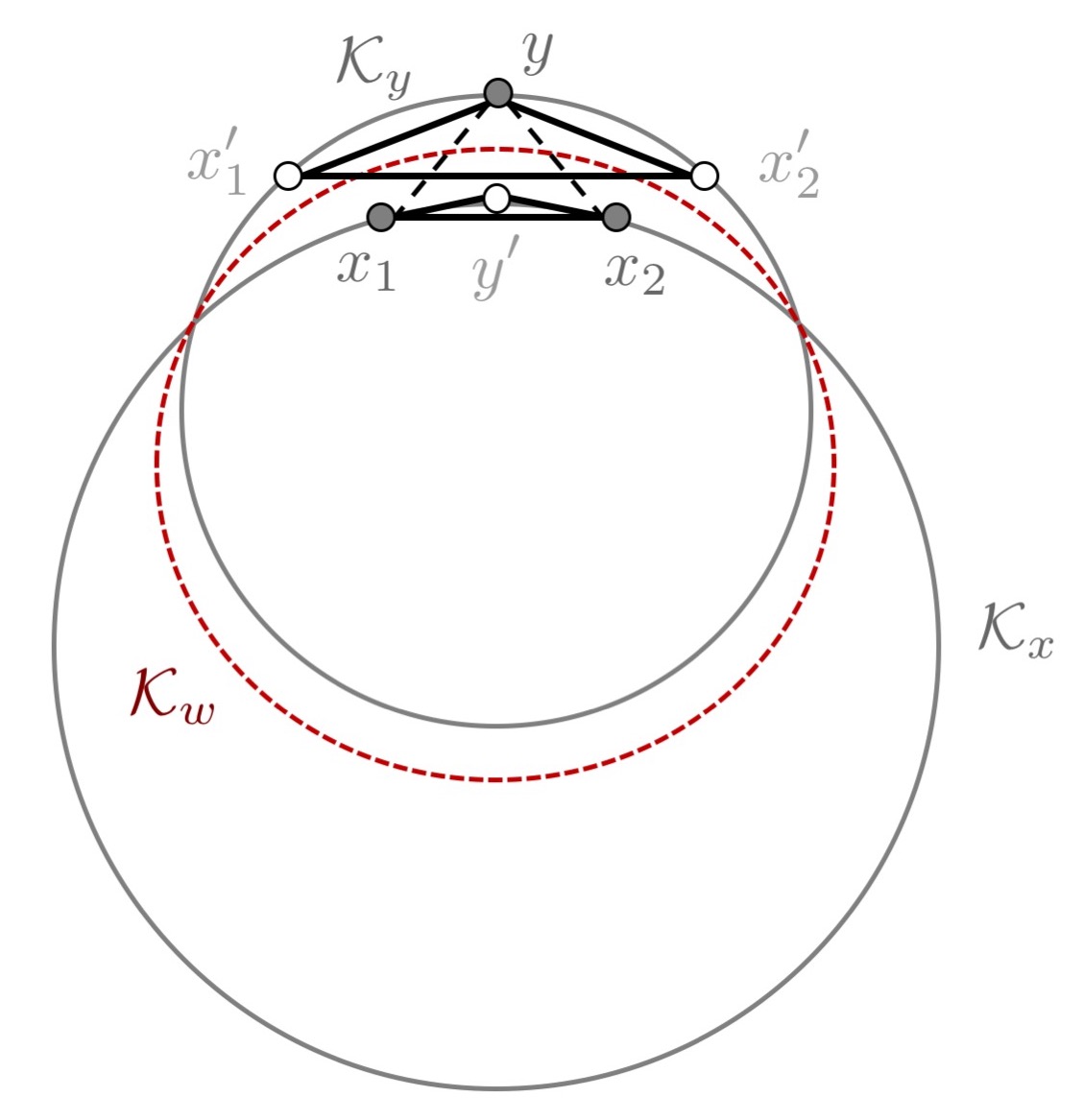}
    \caption{Support vectors on hypercircles $\Kc_x$ and $\Kc_y$ with decision hypercircle $\Kc_w$.}
    \label{fig:A-hyp-1}
\end{figure}
\begin{figure}[ht]
    \centering
    \includegraphics[width=8cm]{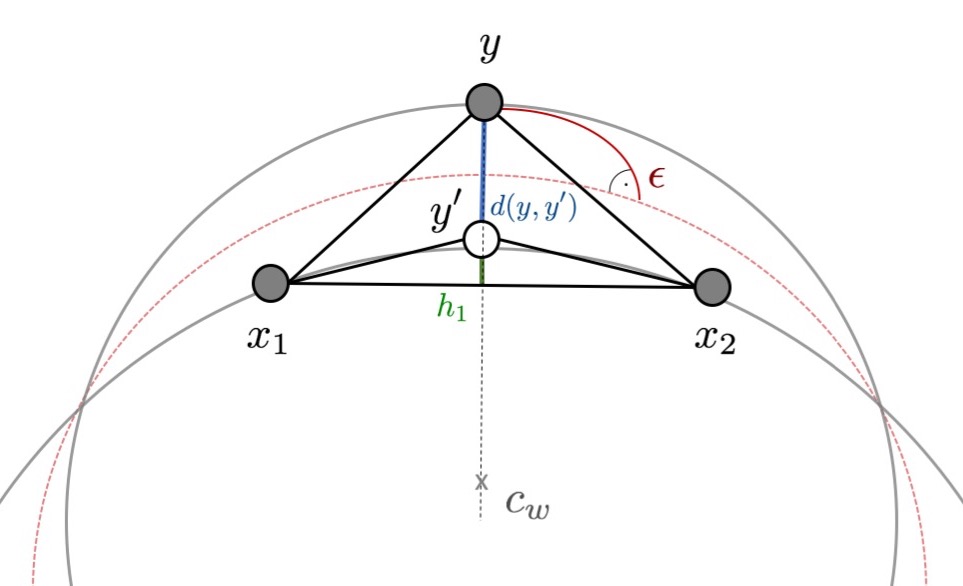}
    \caption{Margin as distance between hypercircles $\Kc_x$ and $\Kc_y$.}
    \label{fig:A-hyp-2}
\end{figure}
\begin{figure}
    \centering
    \includegraphics[width=7cm]{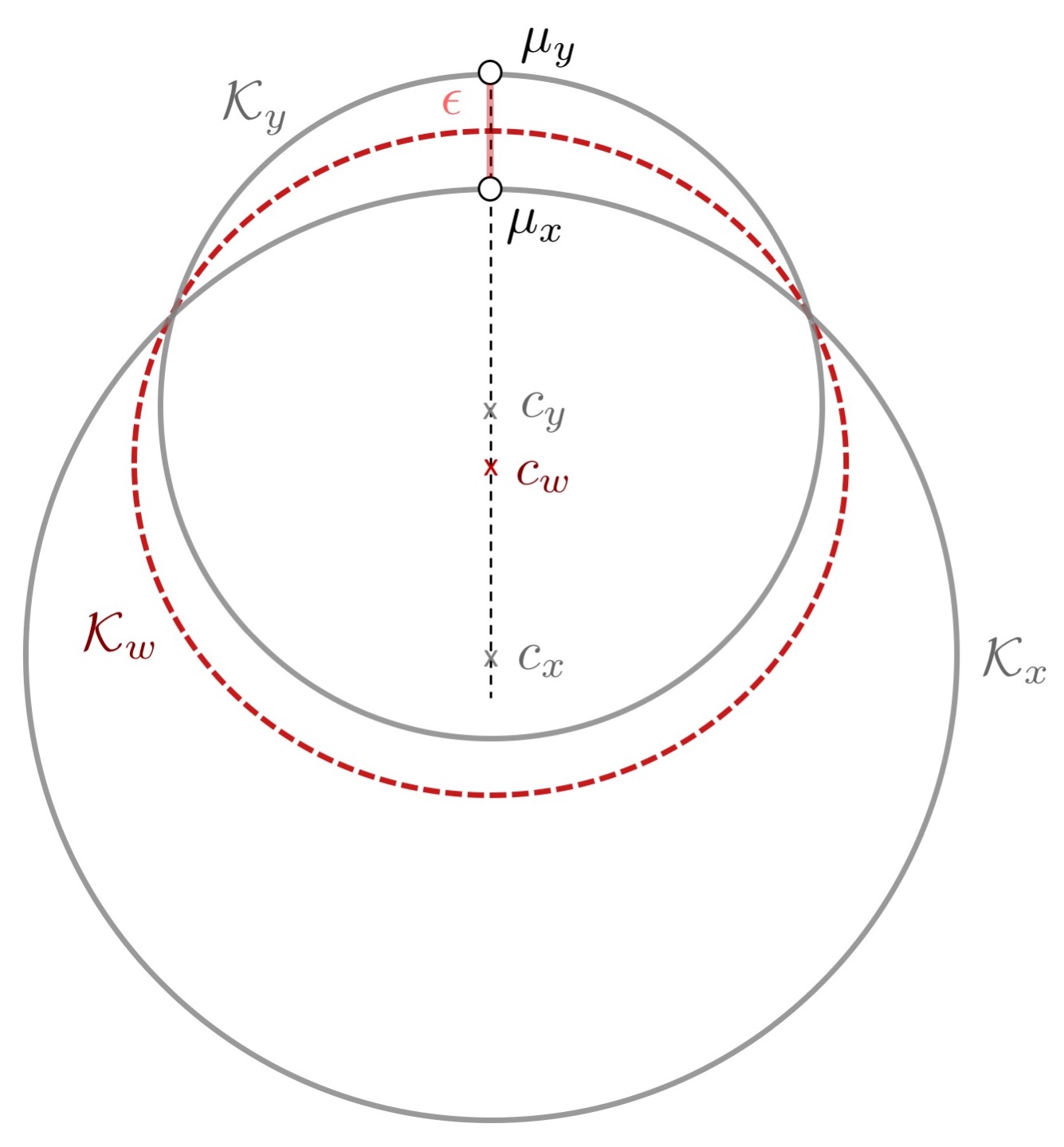}
    \caption{Geometric construction for computing the center and radius of the hypercircle $\Kc_x$.}
    \label{fig:A-hyp-3}
\end{figure}
In $\Pb^2$ the decision hyperplane corresponding to $\hat{w}=Bw$ corresponds to a hypercircle $\Kc_w$. One can show, that its radius is given by $r_w = \sqrt{\frac{1-\lambda}{1+\lambda}}$~\citep{cho}, by computing the hyperbolic distance between a point on the decision boundary and one of the hypercircle's ideal points. Further note, that the support vectors lie on hypercircles $\Kc_x$ and $\Kc_y$, which correspond to the set of points of hyperbolic distance $\epsilon$ (i.e., the margin) from the decision boundary. We again assume wlog that at least one support vector is not unique and let $x_1, x_2 \in \Kc_x$ and $y \in \Kc_y$ (see Fig.~\ref{fig:A-hyp-1}). \\

\begin{theorem}[Thm.~\ref{thm:h-dist-dim}]
$\epsilon' \approx \epsilon$.
\end{theorem}
\begin{proof}
Our proof consists of three steps:\\

\noindent \textbf{Step 1: Find Euclidean radii and centers of hypercircles.}
The hypercircles $\Kc_x, \Kc_y$ correspond to arcs of Euclidean circles $\bar{\Kc}_x, \bar{\Kc}_y$ in the full plane that are related through circle inversion on the decision circle $\bar{\Kc}_w$ (i.e., the Euclidean circle corresponding to $\Kc_w$); see Fig.~\ref{fig:A-hyp-1}.
We can construct a "mirror point" $y' \in \bar{\Kc}_x$ of $y$ by circle inversion on $\bar{\Kc}_w$. We have the following (Euclidean) distance relations: The circle inversion gives
\begin{align*}
    \bar{d}(y',\bar{c}_w) \; \bar{d}(y,\bar{c}_w) &= r_w^2 \; ,
\end{align*}
where $\bar{c}_w$ denotes the center of $\bar{\Kc}_w$. Furthermore, we have (see Fig.~\ref{fig:A-hyp-3})
\begin{align*}
\bar{d}(y,\bar{c}_w) = \bar{d}(y',\bar{c}_w) + \bar{d}(y,y') \; .
\end{align*}
 Putting both together, we get an expression for the Euclidean distance of $y$ and $y'$:
\begin{align*}
    \circled{1} \quad \bar{d}(y,y') &= \bar{d}(\bar{c}_w,y) - \frac{\bar{r}_w^2}{\bar{d}(\bar{c}_w,y)}
    \; .
\end{align*}
Here, we have by construction $\bar{c}_w=(0,a,0,\dots,0)$ with a free parameter $a$. Wlog, assume $\bar{c}_w=(0,-1,0,\dots,0)$. Next, consider the triangle $\Delta(x_1,x_2,y)$. We can express its altitude $h$ in terms of the side length $\bar{d}(x_1,x_2)=: d_1$, $\bar{d}(x_1,y)=:d_2$ and $\bar{d}(x_2,y)=: d_3$ via Heron's formula:
\begin{align*}
    h &= \frac{2}{d_1} \sqrt{s(s-d_1)(s-d_2)(s-d_3)} \; ,
\end{align*}
where $s=\frac{1}{2}(d_1 + d_2 + d_3)$. 
Now, consider the triangle $\Delta(x_1,x_2,y')$. Due to the relation between $y$ and $y'$ in $\circled{1}$, its altitude $h_x$ is related to $h$ as 
\begin{align*}
    \circled{2} \quad h_x = h - \bar{d}(y,y') \; .
\end{align*}
With the side length - altitude relations given in $\Delta(x_1,x_2,y)$ and $\circled{2}$, we can compute the length of the other sides $\bar{d}(x_1,y')$ and $\bar{d}(x_2,y')$ as follows (with Pythagoras theorem):
\begin{align*}
  \bar{d}(x_1,y') &= \left( h_x^2 + \bar{d}(x_1,y)^2 - h^2 \right)^{1/2} \\
  \bar{d}(x_2,y') &= \left( h_x^2 + \bar{d}(x_2,y)^2 - h^2 \right)^{1/2} \;.
\end{align*}
With that, we can compute the radius of $\bar{\Kc}_x$ as follows: $\bar{\Kc}_x$ circumscribes $\Delta(x_1,x_2,y')$, therefore its radius $\bar{r}_x$ can be computed via Heron's formula as
\begin{align*}
    \bar{r}_x &= \frac{\bar{d}(x_1,y')+\bar{d}(x_2,y')+\bar{d}(x_1,x_2)}{4A} \\
    A &= \sqrt{s(s-\bar{d}(x_1,y'))(s-\bar{d}(x_2,y'))(s-\bar{d}(x_1,x_2))}
\end{align*}
where $s=\frac{1}{2}(\bar{d}(x_1,y')+\bar{d}(x_2,y')+\bar{d}(x_1,x_2))$. With an analog construction, we can compute the radius $\bar{r}_y$ of $\Kc_y$ as function of $\bar{d}(x_1',x_2')$, $\bar{d}(x_1',y)$ and $\bar{d}(x_2',y)$ via relations in the triangle $\Delta(x_1',x_2',y)$.\\

\noindent\textbf{Step 2: Express h-margin as distance between hypercircles.}
As shown in Fig.~\ref{fig:A-hyp-2}, the margin is the hyperbolic distance from a point on $\Kc_x, \Kc_y$ to $\Kc_w$, corresponding to the length of a geodesic connecting the point with the closest point on $\Kc_w$. Let $v \in \Kc_x$ and $u \in \Kc_w$ the closest point on the decision circle. From the geometry of the Poincare half plane we know that there exists a M\"obius transform $\theta \in \text{M\"ob}(\Pb^2)$ such that the images $\theta(u)=i\mu$ and $\theta(v)=i\nu$ of $u,v$ lie on the positive imaginary axis. Since the hyperbolic distance is invariant under M\"obius transforms, we get
\begin{align*}
    d(u,v) = d(\theta(u),\theta(v)) = d(i\mu, i\nu) = \Big\vert \log \frac{\nu}{\mu} \Big\vert \; . 
\end{align*}
Similarly, we can express the distance between between support vectors $x \in \Kc_x$ and $y \in \Kc$, which is twice the hyperbolic margin: Let $\theta(x)=i\mu_x$ and $\theta(y)=i\mu_y$, where $\mu_x, \mu_y$ are given by the intersection points of $\Kc_x, \Kc_y$ with the imaginary axis. Then
\begin{align*}
    2 \epsilon = d(x,y) = \Big\vert \log \frac{\mu_y}{\mu_x} \Big\vert \; . 
\end{align*}
We can express $\mu_x, \mu_y$ in terms of the centers and radii of $\Kc_x,\Kc_y$ as follows (Fig.~\ref{fig:A-hyp-2})
\begin{align*}
    \mu_x &= \bar{c}_x^{(2)} + \bar{r}_x \\
    \mu_y &= \bar{c}_y^{(2)} + \bar{r}_y \; ,
\end{align*}
where $c^{(2)}$ denotes the second coordinate of the point $c \in \Pb^m$.
Putting everything together, we get the following expression for the margin:
\begin{align*}
    \circled{3} \quad \epsilon = \frac{1}{2} \Big\vert   \log \frac{\bar{c}_y^{(2)} + \bar{r}_y}{\bar{c}_x^{(2)} + \bar{r}_x}   \Big\vert \; . \\
\end{align*}

\noindent\textbf{Step 3: Evaluate Distortion.} 
As discussed above (Prop.~\ref{prop:e-dist}), the influence of distortion on the altitude $h$ in the triangle $\Delta(x_1,x_2,y)$ is given by the factor $\frac{1}{c_H}$. 
\begin{align*}
    \circled{4} \quad h' = \frac{h}{c_H} \; .
\end{align*}
$\bar{r}_x$ depends on pairwise distances between support vectors and $h$, which are distorted by a factor $\frac{1}{c_H}$ (by assumption on $\phi_H$ and $\circled{4}$). $\bar{r}_x$ depends further on $h_x$ which in turn depends on $\bar{d}(c_w,y)$. The latter depends on the Euclidean norm of the support vector $y$, i.e., $\norm{y}$. With Ass.~\ref{ass:dist} the total multiplicative distortion is then at most of a factor $\frac{1}{c_H}$. 
We can derive an analogue result for $\bar{r}_y$. 
For the center $\bar{c}_x$ note the following:
\begin{align*}
  \bar{c}_x^{(2)} = \frac{1}{2}\left[
    (1-\bar{r}_w^2)\tilde{x}_0 - (1+\bar{r}_w^2)\tilde{x}_1
  \right] \; ,
\end{align*}
where $(\tilde{x}_0, \tilde{x}_1, 0, \dots, 0)=\tilde{x}=(\pi_{BP}\circ(\pi_{LB} \circ B))$ and $\bar{r}_w = \sqrt{\frac{1-\lambda}{1+\lambda}}$. Rewriting
\begin{align*}
    (1-\bar{r}_w^2) \tilde{x}_0 &= \frac{2 \tilde{w}_0 \tilde{x}_0}{\tilde{w}_1 + \tilde{w}_0} \\
    (1+\bar{r}_w^2) \tilde{x}_1 &=\frac{2\tilde{w}_1 \tilde{x}_1}{\tilde{w}_1 + \tilde{w}_0} \; ,
\end{align*}
we get
\begin{align*}
    \bar{c}_x^{(2)} = \frac{\tilde{w}^T \tilde{x}}{\tilde{w}_0+ \tilde{w}_1} \; .
\end{align*}
Similarly, one can derive
\begin{align*}
    \bar{c}_y^{(2)} = \frac{\tilde{w}^T \tilde{y}}{\tilde{w}_0+ \tilde{w}_1} \; ,
\end{align*}
for $(\tilde{y}_0, \tilde{y}_1, 0, \dots, 0)=\tilde{y}=(\pi_{BP}\circ(\pi_{LB} \circ B))$. Both are only affected by distortion of the form (2), i.e. the multiplicative distortion is given by a factor $\frac{1}{c_H}$. Inserting this into the margin expression ($\circled{4}$) gives
\begin{align*}
    \epsilon' &= \frac{1}{2} \Big\vert \log \frac{c_y' + r_y'}{c_x' + r_x'} \Big\vert 
    \gtrsim \frac{1}{2} \Big\vert \log \frac{\frac{c_y}{c_H} + \frac{r_y}{c_H}}{c_H c_x + c_H r_x} \Big\vert
    = \frac{1}{2} \Big\vert \log \left(  \frac{1}{c_H^2} \frac{c_y + r_y}{c_x + r_x}  \right) \Big\vert \\
    &= \frac{1}{2} \Big\vert  \underbrace{\log  \frac{1}{c_H^2}}_{\approx 0} + \log \frac{c_y + r_y}{c_x + r_x} \Big\vert 
    \overset{\dagger}{\approx} \frac{1}{2} \Big\vert \log \frac{c_y + r_y}{c_x + r_x} \Big\vert = \epsilon
    \; ,
\end{align*}   
where ($\dagger$) follows from $c_H=O(1+\epsilon)$ with $\epsilon>0$ small, by Thm.~\ref{thm:sarkar}.
\end{proof}

\section{Additional Experimental Results}
\label{sec:apx-exp}

\subsection{Hyperbolic perceptron}
To validate the hyperbolic perceptron algorithm, we performed two simple classification experiments.
For the two-class data set (ImageNet n09246464 and n07831146), we
observe that hyperbolic perceptron can successfully classify the points into the two groups, i.e., it achieves zero test error. In a second experiment, we try hyperbolic perceptron on a linearly non-separable dataset. The algorithm was still able to classify reasonably well.

\subsection{Adversarial Gradient descent}
\subsubsection{Choice of loss function}
Following the large body of work on large-margin learning in Euclidean space, we tested our approach with the classic hinge (Eq.~\ref{eq:cho-hinge}) and least squares losses (Eq.~\ref{eq:square-loss}). 
While both algorithms work well in practise (see \S~\ref{sec:exp} and Section~\ref{sec:exp-apx}), they do not fulfill Ass.~\ref{ass1} on the whole domain. Therefore, our theoretical guarantees are not valid for those loss functions.

We derive theoretical results for the hyperbolic logistic loss (Eq.~\ref{eq:hyp-log-loss}) instead, which fulfills Ass.~\ref{ass1}. Unfortunately, the hyperparameter $R_x$ is difficult to determine in practice. We therefore decided to omit validation experiments with the hyperbolic logistic loss.

For choosing an \emph{adversarial budget} $\alpha$ in practice, note that Assumption~\ref{ass1}(2) imposes a norm constraint on the adversarial examples, relative to the maximal norm of the training points. Given the constant $R_x$, one can estimate an upper bound on $\alpha$. In addition, an upper bound on $\alpha$ depends on how separable the data set is, i.e., the maximal possible margin. Within these constraints, the choice of $\alpha$ is guided by a trade-off between better robustness and longer training time.

\subsubsection{Adversarial GD via least squares loss}
\label{sec:exp-apx}
Using the same data set as described in \S~\ref{sec:exp}, we also try classification in hyperbolic space with adversarial examples using the least squares losses (Eq.~\ref{eq:square-loss}).
We use the same procedure to find adversarial examples.
The results are plotted in Figure~\ref{fig:adv-gd-sq} with similar conclusions.

\begin{figure*}
   \vspace{-2mm}
    \centering
    \hfill
    \includegraphics[width=0.3\textwidth]{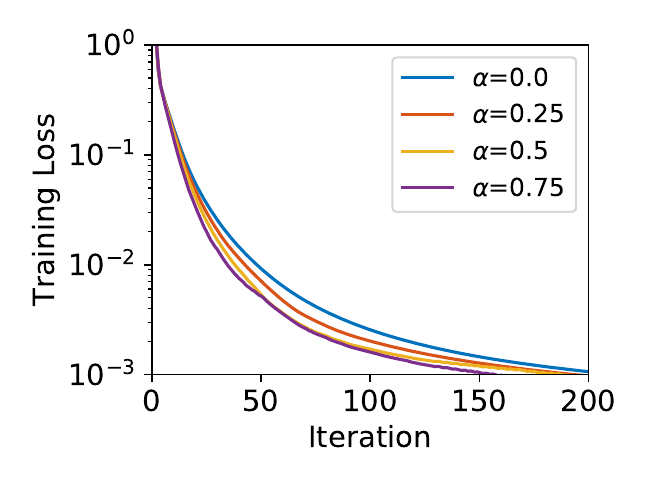}
    \hfill
    \includegraphics[width=0.3\textwidth]{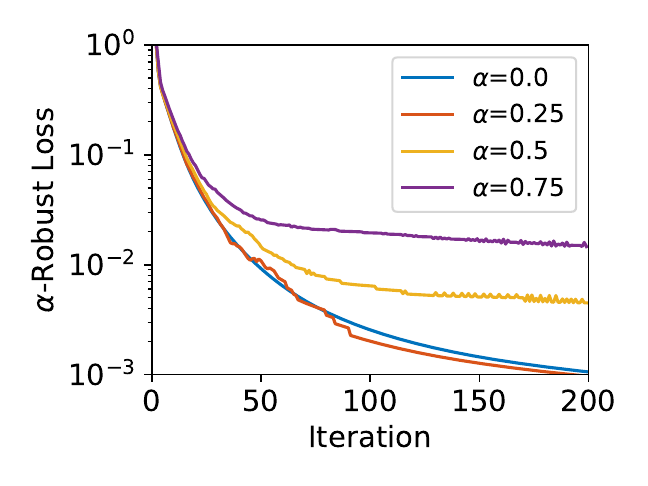}
    \hfill
    \includegraphics[width=0.3\textwidth]{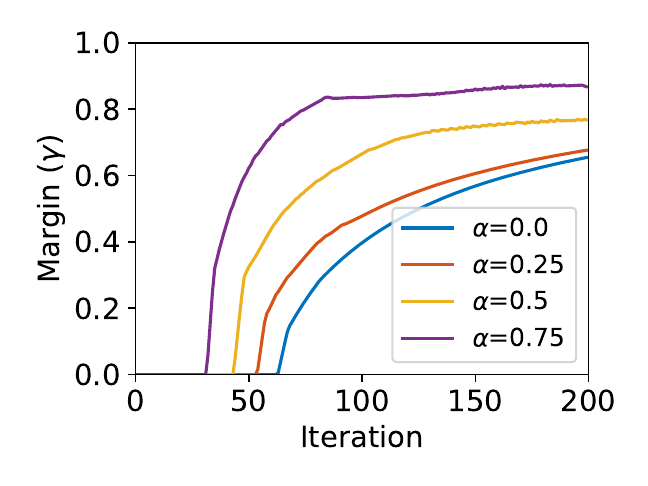}
    \hfill
    \vspace{-3mm}
    \caption{Performance of Adversarial GD using smoothed square loss (Eq.~\ref{eq:square-loss}). \textbf{Left:} Loss $L(\vw)$ on the original data. \textbf{Middle:} $\alpha$-robust loss $L_\alpha(\vw)$. \textbf{Right:} Hyperbolic margin $\gamma_H$.
We vary the adversarial budget $\alpha$ over $\{0, 0.25, 0.5, 0.75\}$. The case $\alpha=0$ corresponds to the setup in~\citep{cho}.}
    \label{fig:adv-gd-sq}
\end{figure*}

\subsection{Dimension-distortion trade-off}\label{apx:f-3}
Euclidean embeddings computed using implementation in~\citet{NK17} by Facebook Research\footnote{\url{https://github.com/facebookresearch/poincare-embeddings}}.
\begin{table}[H]
    \centering
    \begin{tabular}{@{}lcc@{}}
    \toprule
    $d$ & Euclidean & Hyperbolic \\
    \midrule
    4 & 0.54 & 0.51 \\
    8 & 0.53 & 1.00 \\
    16 & 0.68 & 1.00 \\ 
    \bottomrule
    \end{tabular}
    \caption{Classification performance (test error) in hyperbolic vs. Euclidean space of dimension $d$.}
    \label{tab:my_label}
\end{table}

\end{document}